\documentclass[nohyperref]{article}

\usepackage{microtype}
\usepackage{graphicx}
\usepackage{subcaption}
\usepackage{booktabs} 

\usepackage{hyperref}

\usepackage[accepted]{configurations/icml2022}

\usepackage{amsmath}
\usepackage{amssymb}
\usepackage{amsthm}
\usepackage{thmtools,thm-restate}
\usepackage{mathtools}

\usepackage[capitalize,noabbrev]{cleveref}

\theoremstyle{plain}

\theoremstyle{definition}

\theoremstyle{remark}

\usepackage[textsize=tiny]{todonotes}

\usepackage{ifthen}
\usepackage{enumitem}
\usepackage{soul}
\usepackage{multirow}
\usepackage{algorithm}
\usepackage{algorithmic}

\usepackage{tikz}
\usetikzlibrary{bayesnet}
\usetikzlibrary{arrows}
\usepackage{caption}
\usetikzlibrary{backgrounds}

\newcommand{\ours}[0]{\textbf{iMSDA}}
\newcommand{\Pb}[1]{{\mathbb{P}}\left[#1\right] }
 
\newcommand{\rvline}{\hspace*{-\arraycolsep}\vline\hspace*{-\arraycolsep}}
\newcommand{\uu}{{\mathbf u}}
\newcommand{\xx}{{\mathbf x}}
\newcommand{\zz}{{\mathbf z}}
\newcommand{\yy}{{\mathbf y}}

\newcommand{\cT}{\mathcal{T}}
\newcommand{\cL}{\mathcal{L}}

\newcommand{\cN}{\mathcal{N}}
\newcommand{\mI}{\mathbf{I}}
\newcommand{\mH}{\mathbf{H}}

\newcommand{\mJ}{\mathbf{J}}

\icmltitlerunning{Partial Identifiability for Domain Adaptation}

\begin{document}

\twocolumn[
\icmltitle{Partial Identifiability for Domain Adaptation}



{\icmlsetsymbol{equal}{*}}

\begin{icmlauthorlist}
\icmlauthor{Lingjing Kong}{cmu}
\icmlauthor{Shaoan Xie}{cmu}
\icmlauthor{Weiran Yao}{cmu}
\icmlauthor{Yujia Zheng}{cmu}
\icmlauthor{Guangyi Chen}{mbzuai,cmu}
\icmlauthor{Petar Stojanov}{broad}
\icmlauthor{Victor Akinwande}{cmu}
\icmlauthor{Kun Zhang}{mbzuai,cmu}
\end{icmlauthorlist}

\icmlaffiliation{cmu}{Carnegie Mellon University, USA}
\icmlaffiliation{mbzuai}{Mohamed bin Zayed University of Artificial Intelligence, UAE}
\icmlaffiliation{broad}{Broad Institute of MIT and Harvard, USA}


\icmlcorrespondingauthor{Kun Zhang}{kunz1@cmu.edu}

\icmlkeywords{Machine Learning, ICML}

\vskip 0.3in
]



\printAffiliationsAndNotice{}  

\begin{abstract}
    Unsupervised domain adaptation is critical to many real-world applications where label information is unavailable in the target domain.
    In general, without further assumptions, the joint distribution of the features and the label is not identifiable in the target domain.  
    To address this issue, we rely on the property of minimal changes of causal mechanisms across domains to minimize unnecessary influences of distribution shift.
    To encode this property, we first formulate the data-generating process using a latent variable model with two partitioned latent subspaces: invariant components whose distributions stay the same across domains and sparse changing components that vary across domains.
    We further constrain the domain shift to have a restrictive influence on the changing components.
    Under mild conditions, we show that the latent variables are partially identifiable, from which it follows that the joint distribution of data and labels in the target domain is also identifiable. 
    Given the theoretical insights, we propose a practical domain adaptation framework called \ours.
    Extensive experimental results reveal that \ours\ outperforms state-of-the-art domain adaptation algorithms on benchmark datasets, demonstrating the effectiveness of our framework.

\end{abstract}

\section{Introduction}
\label{sec:introduction}

Unsupervised domain adaptation (UDA) is a form of prediction setting in which the labeled training data and the unlabeled test data follow different distributions. UDA can frequently be done in settings where multiple labeled training datasets are available, and this is termed as multiple-source UDA. Formally, given features $\xx $, target variables $ \yy $, and domain indices $\uu$, the training (source domain) data follows multiple joint distributions $p_{\xx, \yy | \uu_{1}}$, $p_{\xx, \yy| \uu_{2}}$, ..., $p_{\xx, \yy|\uu_{M}}$,\footnote{
    With a slight abuse of notation, we use $p_{ \xx, \yy | \uu' } $ to represent the conditional distribution for a \textit{specific} domain $\uu'$, i.e.\ $p_{ \xx, \yy | \uu } (\xx, \yy | \uu' ) $.
} and the test (target domain) data follows the joint distribution $p_{\xx, \yy|\uu^{\cT}}$, where $p_{\xx, \yy | \uu} $ may vary across $\uu_{1}$, $\uu_{2}$, ..., $\uu_{M} $. During training, for each $i$-th source domain, we are given labeled observations $(\xx_{k}^{(i)}, \yy_{k}^{(i)})_{k=1}^{m_i}$ from $p_{\xx, \yy| \uu_{i}}$, and target domain unlabeled instances $ (\xx^{\cT}_{k})_{k=1}^{m_T}$ from $p_{\xx, \yy|\uu^{\cT}}$. The main goal of domain adaptation is to make use of the available observed information to construct a predictor that will have optimal performance in the target domain. \par 
It is apparent that without further assumptions, this objective is ill-posed. Namely, since the only available observations in the target domain are from the marginal distribution $p_{\xx|\uu^{\cT}}$, the data may correspond to infinitely many joint distributions $p_{\xx, \yy|\uu^{\cT}}$. This mandates making additional assumptions on the relationship between the source and the target domain distributions, with the hope of being able to reconstruct (i.e., identify) the joint distribution in the target domain $p_{\xx, \yy|\uu^{\cT}}$. Typically, these assumptions entail some measure of similarity  across all of the joint distributions. One classical assumption is that each joint distribution shares the same conditional distribution $p_{\yy|\xx}$, whereas $p_{\xx | \uu}$ changes, more widely known as the covariate shift setting \cite{Shimodaira00,Sugiyama08,Zadrozny04,Huang07,cortes10}. In addition, assumptions can be made that the source and the target domain distributions are related via the changes in $p_{\yy|\uu}$ or one or more linear transformations of the features $\xx$~\cite{zhang2013domain,zhang2015multi,gong2016domain}. \par 
In order to drop any parametric assumptions regarding the relationship between the domains, one may want to consider a more abstract assumption of \textit{minimal changes}. To make reasoning about this notion easier, it is often useful to consider the data-generating process \cite{Scholkopf12,zhang2013domain}. For example, if the data-generating process is $Y \rightarrow X$, then $p_{\yy | \uu}$ and $p_{\xx|\yy, \uu}$ change independently across domains, and factoring the joint distribution in terms of these factors makes it possible to represent its changes in the most parsimonious way. Furthermore, the changes of the distribution $p_{\xx|\yy, \uu}$ can be represented as lying on a low-dimensional manifold \cite{stojanov2019data}. The generating process of the observed features and its changing parts across domains can also be automatically discovered from multiple-domain data \cite{zhang2017causal,Huang_JMLR20}, and the changes can be captured using conditional GANs \cite{zhang2020domain}.  \par 
With the advent of deep learning capabilities, another approach is to harness these deep architectures to learn a feature map to some latent variable $\zz$, such that $p_{\zz|\uu_{1}} = p_{\zz|\uu_{2}}=...=p_{\zz|\uu_{M}}=p_{\zz|\uu^{\cT}}$ (i.e., \textit{marginally} invariant), while training a classifier $f_{\text{cls}}: \mathcal{Z} \rightarrow \mathcal{Y}$ on the labeled source domain data, which ensures that $\zz$ has predictive information about $\yy$ \cite{ganin2014unsupervised,ben2010theory,zhao2018adversarial}. This could still mean that the joint distributions $p_{\zz, \yy|\uu}$ may be different across domains $\uu_{1}$, ..., $\uu_{M}$, leading to potentially poor prediction performance in the target domain. Solving the problem has been attempted using deep generative models and enforcing disentanglement of the latent representation \cite{lu2020invariant, cai2019learning}. However, none of these efforts establishes identifiability of $p_{\xx, \yy|\uu^{\cT}}$.  \par 
Since deep architectures are a requirement for strong performance on high-dimensional datasets, it has become essential to combine their representational power with the notion of minimal changes in the joint distribution across domains. In fact, one can make reasonable assumptions on the data-generating process and represent the minimal changes of $p_{\xx, \yy|\uu}$ across domains such that both $\zz$ and $p_{\xx,\yy|\uu}$ are identifiable in the target domain. \par 
In this paper, we represent the variables $\xx$, $\yy$, and $\zz$ in terms of a data-generating process assumption. To enforce \textit{minimal change} of the joint distribution $p_{\xx, \yy, \zz|\uu}$ across domains, we allow only a partition of the latent space to vary over domains and constrain the flexibility of the influence from domain changes.
(1) Under this generating process, we show that both the changing and the invariant parts in the generating process are (partially) identifiable, which gives rise to a principled way of aligning source and target domains. (2) Leveraging the identifiability results, we propose an algorithm that identifies the latent representation and utilizes the representation for optimal prediction for the target domain. (3) We show that our algorithm surpasses state-of-the-art UDA algorithms on benchmark datasets.

\section{Related Work}

Domain adaptation \cite{cai2019learning, Courty17, deng2019cluster, tzeng2017adversarial, wang2019easy, xu2019larger, zhang2017joint, zhang2013domain, zhang2018collaborative, zhang2019domain, mao2021source, stojanov2021domain, wu2022stylealign, eastwood2022sourcefree,tong2022adversarial,zhu2022mind,xu2022cdtrans,berthelot2022adamatch,kirchmeyer2022mapping} is an actively studied area of research in machine learning and computer vision. 
One popular direction is invariant representation learning, introduced by \cite{ganin2014unsupervised}. Specifically, it has been successively refined in several different ways to ensure that the invariant representation $\zz$ contains meaningful information about the target domain. For instance, in \cite{bousmalis2016domain}, $\zz$ was divided into a changing and an invariant part, and a constraint was added that $\zz$ has to be able to reconstruct the original data, in addition to predicting $\yy$ (from the shared part). Furthermore, pseudo-labels have also been utilized in order to match the joint distribution $p_{\zz,\yy}$ across domains, using kernel methods \cite{long2017conditional,long2017deep}. Another approach to ensure more structure in $\zz$ is to assume that the class-conditional distribution $p_{\xx|\yy}$ has a clustering structure and regularize the neural network algorithm to prevent decision boundaries from crossing high-density regions \cite{shu2018dirt,xie2018learning}. In these studies, $p_{\yy}$ was typically assumed to stay the same across domains. Dropping this assumption has been investigated in \cite{wu2019domain, guo2020ltf}. Furthermore, the setting in which the label sets do not exactly match has also been studied \cite{saito2020universal}. While the body of work studying representation learning for domain adaptation is large and extensive, there are still no guarantees that the learned representation will help learn the optimal predictor in the target domain, even in the infinite-sample case. Besides, invariant representation learning also sheds light on adaptation with multiple sources~\cite{xu2018deep, peng2019moment, wang2020learning, li2021t, park2021information}.\par 
Along the line of work of generative models, studies on generative models and unsupervised learning have made headway into better understanding the properties of identifiability of a latent representation which is meant to capture some underlying factors of variation from which the data was generated. Namely, independent component analysis (ICA) is the classical approach for learning a latent representation for which there are identifiability results, where the generating process is assumed to consist of a linear mixing function \cite{comon1994independent,bell1995information,ICAbook}. A major problem in nonlinear ICA is that, without assumption on either source distribution or generating process, the model is seriously unidentifiable \cite{hyvarinen1999nonlinear}. Recent breakthroughs introduced auxiliary variables, e.g., domain indexes, to advance the identifiability results \cite{hyvarinen2019nonlinear, sorrenson2020disentanglement, halva2020hidden, lachapelle2021disentanglement, khemakhem2020variational, von2021self, lu2020invariant}. These works aim to identify and disentangle the components of the latent representation while assuming that all of them are changing across domains. \par 
In the context of self-supervised learning with data augmentation, \citet{von2021self} considered the identifiability of shared part, which stays invariant between views. However, this line of work assumes the availability of paired instances in two domains. In the context of out-of-distribution generalization, \citet{lu2020invariant} extend the identifiability result of iVAE~\cite{khemakhem2020variational} to a general exponential family that is not necessarily factorized. However, this study does not take advantage of the fact that the latent representations should contain invariant information that can be disentangled from the part that corresponds to changes across domains. Most importantly, this study resorts to finding a conditionally invariant sub-part of $\zz$, even though there may be parts of $\zz$ that are not conditionally invariant and yet still relevant for predicting $\yy$.  \par 
In this paper, we make use of realistic assumptions regarding the data-generating process in order to provably identify the changing and invariant aspects of the latent representation in both the source and target domains. In doing so, we drop any parametric assumptions about $\zz$, and we allow both the changing and invariant parts to have predictive information about $\yy$. We show that we can align domains and make predictions in a principled manner, and we present an autoencoder algorithm to solve the problem in practice.

\section{High-level Invariance for Domain Adaptation}
\label{sec:motivation}

\begin{figure}
    \centering
    \begin{tikzpicture}[object/.style={thin,double,<->}]
      \node[obs] (x) {$\mathbf{x}$};%
      \node[obs,right=0.75cm of x] (y) {$\mathbf{y}$};%
     \node[latent,above=0.5cm of x,xshift=-0.75cm] (zc) {$\mathbf{z}_{c}$}; %
     \node[latent,above=0.5cm of x,xshift=0.75cm] (zs) {$\mathbf{z}_{s}$}; %
     \node[obs,above=0.5cm of zs,xshift=-0.75cm] (u) {$\mathbf{u}$}; %
     \node[latent,above=0.5cm of zs,xshift=0.75cm] (zs_tilde) {$\tilde{\mathbf{z}}_{s}$}; %
     \edge {zc,zs} {x} 
     \edge {u,zs_tilde} {zs}
     \edge {zc,zs_tilde}{y}
    \end{tikzpicture}
    \caption{\textbf{The generating process}: The gray shade of nodes indicates that the variable is observable.}
    \label{fig:generating_process}
    \end{figure}
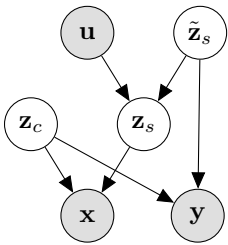
    
    In this section, we introduce our data-generating process in Figure~\ref{fig:generating_process} and Equation~\ref{eq:data_generating_process} and discuss how we could exploit this latent variable model to handle UDA.

    \begin{align} \label{eq:data_generating_process}
        \begin{split}
            \zz_{c} \sim p_{ \zz_{c} }, \; 
            \tilde{\zz}_{s} \sim p_{ \tilde{\zz}_{s} }, \; 
            \zz_{s} = f_{\uu} ( \tilde{ \zz }_{s} ), \;
            \xx = g( \zz_{c}, \zz_{s} ).
        \end{split}
    \end{align}
    
    In the generating process, we assume that data $\xx \in \mathcal{X} $ (e.g.\ images) are generated by latent variables $\zz \in \mathcal{Z} \subseteq \mathbb{R}^{n}$ through an invertible and smooth mixing function $g: \mathcal{Z} \to \mathcal{X}$.
    We denote by $\uu \in \mathcal{U}$ the domain embedding, which is a constant vector for a specific domain.

    We partition latent variables $\zz$ into two parts $ \zz = [ \zz_{c}, \zz_{s} ] $: the invariant part $ \zz_{c} \in \mathcal{Z}_{c} \subseteq \mathbb{R}^{n_{c}}$ (i.e.\ content) of which the distribution stays constant over domain $\uu$'s, and the changing part $ \zz_{s} \in \mathcal{Z}_{s} \subseteq \mathbb{R}^{n_{s}} $ (i.e.\ style) with varying distributions over domains.

    We parameterize the influence of domain $\uu$ on $\zz_{s}$ as a simple transformation of some generic form of the changing part, given by $\tilde{\zz}_s$. Namely, given a component-wise monotonic function $f_{\uu}$, we let $ \zz_{s} = f_{\uu} (\tilde{\zz}_{s}) $. For example, in image datasets, $\zz_{s}$ can correspond to various kinds of backgrounds from the images (sand, trees, sky, etc.). In this case, $\tilde{\zz}_s$ can be interpreted as a generic background pattern that can easily be transformed into a domain-specific image background, depending on which function $f_{\uu}$ is used. \par 
    Further, we assume that $ \yy $ is generated by invariant latent variables $\zz_{c}$ and $\tilde{\zz}_{s}$. Thus, this generating process addresses the conditional-shift setting, in which $ p_{ \xx | \yy, \uu } $ changes across domains, and $p_{ \yy | \uu} = p_{\yy}$ stays the same. 

    We describe below the distinguishing features of our generating process and illustrate how these features are essential to tackling UDA.

    \paragraph{Partitioned latent space.} 
    As discussed in Section~\ref{sec:introduction}, the bulk of prior work \cite{ganin2014unsupervised,ben2010theory,zhao2018adversarial} focuses on learning an invariant representation over domains so that one classifier can be applied to novel domains in the latent space.
    Unlike these approaches, we will demonstrate that our parameterization of $\zz_s$ allows us to preserve the information of the changing part $\zz_{s}$ for prediction instead of discarding this part altogether by imposing invariance.
    Similarly, recent work \cite{lu2020invariant} allows for the possibility that domain changes could influence all latent variables in their framework to address domain shift but discards a subset of them that may be relevant for predicting $\yy$. In contrast, our goal is to disentangle the changing and invariant parts $\zz_s$ and $\zz_c$, capture the relationship between $\zz_s$ and $\yy$ across domains by learning $f_{\uu}$, and use this information to perform prediction in the target domain. \par
    Our parameterization also allows us to implement the minimal change principle by constraining the changing components to be as few as possible. Otherwise, unnecessarily large domain influences (e.g., all components being changing) may lead to loss of semantic information in the invariant $ (\zz_{c}, \tilde{\zz}_{s}) $ space. For instance, in a scenario where each domain comprises digits $0$-$9$ with a specific rotation degree, unconstrained influence may result in a mapping between $1$ and $9$ across domains.
    
    \paragraph{High-level Invariance $\tilde{\zz}_{s}$.} 
    We note that the presence of $\tilde{\zz}_{s}$ is significant, as it allows us to provably learn an optimal classifier over domains without requiring that $ p_{\zz|\uu} $ to be invariant over domains as in previous work~\cite{ganin2014unsupervised,ben2010theory,zhao2018adversarial}.
    With the component-wise monotonic function $f_{\uu}$, we are able to identify $ \tilde{\zz}_{s} $ through $ \zz_{s} $ which is critical to our ability to learn how the changes across domains are related to each other.
    If we know $f_\uu$ and $\zz_s$, we can always map all domains to this high-level $\tilde{\zz}_s$, which will have the desired information to predict $\yy$ (in addition to $\zz_{c}$).
    Additionally, this restrictive function form also contributes to regulating unnecessary changes. 
    
    The problem remains whether it is possible to identify the latent variables of interest (i.e., $ \zz_{c} $, $ \tilde{\zz}_{s} $) with only observational data $ (\xx, \uu)$. Because if this is the case, we can leverage labeled data in source domains to learn a classifier according to $ p_{ \yy | \tilde{\zz}_{s}, \zz_{c} } $ that is universally applicable over given domains.
    In Section~\ref{sec:theory}, we show that we can indeed achieve this identifiability for this generating process and describe how such a model can be learned.
    In Section~\ref{sec:implementation}, we present the corresponding algorithmic implementation based on variational auto-Encoder (VAE)~\cite{kingma2014autoencoding}.

\section{Partial Identifiability of the Latent Variables}
\label{sec:theory}

In this section, we show our theoretical results that serve as the foundation of our algorithm.
In addition to guiding algorithmic design, we highlight that our theory is a novel contribution to the deep generative modeling literature.
In Section~\ref{subsec:changing_components_identifiability} and Section~\ref{subsec:invariant_subspace_identifiability}, we show that with unlabeled observational data $ (\xx, \uu) $ over domains, we can partially identify the latent space.
In Section~\ref{subsec:joint_distribution_identifiability}, we discuss how these identifiability properties, together with labeled source domain data, yield a classifier useful to the target domain. \par
We parameterize $ \zz_{s} = f_{\uu} ( \tilde{\zz}_{s} ) $ with a component-wise monotonic function for each domain $\uu$, and learn a model $ ( \hat{g}, \, p_{\hat{\zz}_{c}}, \, p_{\hat{\zz}_{s}|\uu} ) $ \footnote{
    We use $\hat{\cdot}$ to distinguish estimators from the ground truth.
} that assumes the identical process as in Equation~\ref{eq:data_generating_process} and matches the true marginal data distribution in all domains, that is,
\begin{align} \label{eq:estimated_condition}
    p_{ \xx | \mathbf{u} } (\xx' | \uu') = p_{ \hat{\xx} | \mathbf{u} } (\xx' | \uu'), \; \quad \forall \xx' \in \mathcal{X}, \, \mathbf{u}' \in \mathcal{U},
\end{align}
where the random variable $ \mathbf{x} $ is generated from the true process $ (g, p_{\mathbf{z}_{c}}, p_{\mathbf{z}_{s} | \mathbf{u}}) $ and $ \hat{\mathbf{x}} $ is from the estimated model $ ( \hat{g}, \, p_{\hat{\zz}_{c}}, \, p_{\hat{\zz}_{s} | \mathbf{u} }) $.

In the following, we show that our estimated model can recover the true changing components and the invariant subspace $\zz_{c}$ and then discuss the implications of our theory for UDA.

\subsection{Identifiability of Changing Components} \label{subsec:changing_components_identifiability}

First, we show that the changing components can be identified up to permutation and invertible component-wise transformations.
More formally, for each true changing component $ z_{s, i} $, there exists a corresponding estimated changing component $ \hat{z}_{s, j} $ and an invertible function $h_{s, i}: \mathbb{R} \to \mathbb{R}$, such that $ \hat{z}_{s, j} = h_{s, i} ( z_{s, i} ) $.
For ease of exposition, we assume that the $\zz_{c}$ and $\zz_{s}$ correspond to components in $\zz$ with indices $\{1, \dots, n_{c}\}$ and $\{n_{c}+1, \dots, n\}$ respectively, that is, $ \zz_{c} = ( z_{i} )_{i=1}^{n_{c}} $ and $ \zz_{s} = ( z_{i} )_{i=n_{c}+1}^{n} $.

\newcommand{\wvectorcontent}{%
&\mathbf{w}( \mathbf{z}_{s}, \mathbf{u} ) 
        = \Big(
                \frac{\partial q_{n_{c}+1} \left(z_{n_{c}+1}, \mathbf{u} \right)}{\partial z_{n_{c}+1}}, \ldots, \frac{\partial q_{n} \left(z_{n}, \mathbf{u}\right)}{\partial z_{n}},  && \nonumber \\ 
            & \quad \frac{\partial^{2} q_{n_{c}+1}\left( z_{ n_{c} + 1 }, \mathbf{u}\right)}{\partial z_{ n_{c} + 1 }^{2}}, \ldots, \frac{\partial^{2} q_{n}\left(z_{n}, \mathbf{u}\right)}{\partial z_{n}^{2}}  
        \Big) &&
}

\begin{restatable}{theorem}{stheory} \label{thm:changing_part_identifiability}
We follow the data-generation process in Equation~\ref{eq:data_generating_process} and make the following assumptions:
\begin{itemize}
    \item A1 (\ul{Smooth and Positive Density)}:
        The probability density function of latent variables is smooth and positive, i.e., $ p_{\mathbf{z}|\mathbf{u}}  $ is smooth and $ p_{\mathbf{z}|\mathbf{u}} > 0 $ over $ \mathcal{Z} $ and $ \mathcal{U} $.
    \label{assumption:pdf_conditions}
    
    \item A2 \ul{(Conditional independence)}: Conditioned on $ \mathbf{u} $, each $ z_{i} $ is independent of any other $ z_{j} $ for $ i, j \in [n]$, $ i \neq j $, i.e.\ 
    $ \log p_{\mathbf{z} | \mathbf{u}} (\zz | \uu) = \sum_{i}^{n} q_{i} ( z_{i}, \mathbf{u} ) $ where $q_{i}$ is the log density of the conditional distribution, i.e., $ q_{i} := \log p_{ z_{i} | \mathbf{u} } $.
    \label{assumption:factorizable_prior}
    
    \item A3 \ul{(Linear independence)}: For any $\zz_{s} \in \mathcal{Z}_{s} \subseteq \mathbb{R}^{n_{s}}$, there exist $2n_{s}+1$ values of $\mathbf{u}$, i.e., $\mathbf{u}_{j}$ with $j=0, 1, \dots, 2n_{s}$, such that the $2n_{s}$ vectors $\mathbf{w}(\mathbf{z}_{s}, \mathbf{u}_{j})-\mathbf{w}(\mathbf{z}_{s}, \mathbf{u}_{0})$ with $j=1,...,2{n_s}$, are linearly independent, where vector $\mathbf{w}(\mathbf{z}_{s}, \mathbf{u})$ is defined as follows:
     \begin{flalign} \label{assumption:sufficient_variability_for_changing}
     \wvectorcontent.
    \end{flalign}
\end{itemize}
By learning $ ( \hat{g}, \, p_{\hat{\mathbf{z}}_{c}}, \, p_{ \hat{\zz}_{s} | \mathbf{u} } )$ to achieve Equation~\ref{eq:estimated_condition}, $\zz_{s}$ is component-wise identifiable.
\end{restatable}
A proof can be found in Appendix~\ref{sec:changing_part_proof}. We note that all three assumptions are standard in the nonlinear ICA literature~\cite{hyvarinen2019nonlinear,khemakhem2020variational}.
In particular, the intuition of Assumption A3 is that $ p_{ \zz_{s} | \uu } $ varies sufficiently over domain $\uu$'s.
We demonstrate in Section~\ref{sec:synthetic_experiments} that this identifiability can be largely achieved with Gaussian distributions of distinct means and variances.

It is noteworthy that Theorem~\ref{thm:changing_part_identifiability} manifests the benefits of minimal changes - the smaller $n_{s}$, the more easily for the linear independence condition to hold, as we need only $2n_{s}+1$ domains with sufficient variability, whereas prior work often requires $2n+1$ sufficiently variable domains.
We note that in many DA problems, the relations between domains are explicit, so the dimensionality of the intrinsic change is small. 
Our results provide insights into whether a fixed number of domains is sufficient for DA with different levels of change.

The identifiability of the changing part is crucial to our method. Only by recovering the changing components $\zz_{s}$ and the corresponding domain influence $f_{\uu}$ can we correctly align various domains with $\tilde{\zz}_{s}$ and learn a classifier.


\subsection{Identifiability of the Invariant Subspace}

In this section, we show that the invariant subspace is identifiable in a block-wise fashion.
Our goal is to show that the estimated invariant part has preserved information about the true invariant part, and there is no information mixed in from the changing components.
Formally, we would like to show that there exists an invertible mapping $ h'_{c}: \mathcal{Z}_{c} \to \mathcal{Z}_{c} $ between the estimated invariant part $ \hat{\zz}_{c} $ and the true invariant part $\zz_{c}$, s.t.\ $ \hat{\mathbf{z}}_{c} = h'_{c}( \mathbf{z}_{c} ) $.
The idea of block-wise identifiability emerges in independent subspace analysis~\cite{casey2000separation,hyvarinen2000emergence,le2011learning,theis2006towards}, and recent work~\cite{von2021self} addresses it in the nonlinear ICA scenario.

In the following theorem, we show that the block-wise identifiability of the invariant subspace can be attained with the estimated data-generating process.
Note that block-wise identifiability is a weaker notion of identifiability than the component-wise identifiability achieved for the changing components. 
Thus, we only claim \textit{partial identifiability} for the invariant latent variables. \par

\begin{restatable}{theorem}{ctheory} \label{thm:block_identifiability_content}
    We follow the data generation process in Equation~\ref{eq:data_generating_process} and make Assumption A1-A2.
    \\
    In addition, we make the following assumption:
    \begin{itemize}[leftmargin=*] \itemsep0em
        \item A4 \ul{(Domain Variability)}:
        For any set $ A_{\mathbf{z}} \subseteq \mathcal{Z} $ with the following two properties:
        \begin{itemize}
            \item a) $ A_{\mathbf{z}} $ has nonzero probability measure, i.e.\ $ \Pb{ \{ \mathbf{z} \in A_{\mathbf{z}} \} | \{ \mathbf{u} = \mathbf{u}' \} } > 0 $ for any $ \mathbf{u}' \in \mathcal{U}$.
            \item b) $A_{\mathbf{z}}$ cannot be expressed as $ B_{\mathbf{z}_{c}} \times \mathcal{Z}_{s} $ for any $ B_{ \mathbf{z}_{c} } \subseteq \mathcal{Z}_{c}$.
        \end{itemize}          
        $ \exists \mathbf{u}_{1}, \mathbf{u}_{2} \in \mathcal{U} $, such that
        \begin{align*}
            \int_{ \zz \in A_{\mathbf{z}} } p_{\mathbf{z} | \mathbf{u}} ( \mathbf{z} | \mathbf{u}_{1} ) d\zz \neq \int_{ \zz \in A_{\mathbf{z}} } p_{\mathbf{z} | \mathbf{u}} ( \mathbf{z} | \mathbf{u}_{2} ) d\zz.
        \end{align*}
        \label{assumption:sufficient_variability_for_invariant_part}
    \end{itemize}
    With the model $ ( \hat{g}, \, p_{\hat{\zz}_{c}}, \, p_{\hat{\zz}_{s}|\uu} ) $ estimated by observing Equation~\ref{eq:estimated_condition}, the $ \mathbf{z}_{c} $ is block-wise identifiable.  
\end{restatable}
A proof is deferred to Appendix~\ref{sec:proof_invariant_part}.

Like Assumption A3, Assumption A4 also requires the distribution $ p_{\zz|\uu} $ to change sufficiently over domains.
Intuitively, the chance of having one such subset $A_{\zz}$ on which \textit{all} domain distributions have an equal probability measure is very slim when the number of domains is large, and the domain distributions differ sufficiently.
In Section~\ref{sec:synthetic_experiments}, we verify Theorem~\ref{thm:block_identifiability_content} with Gaussian distributions with varying means and variances. 

We note that in comparison to the identifiability theory in recent work~\citep{von2021self}, our setting is more challenging in the sense that the theory in \citet{von2021self} relies on the pairing of augmented data, whereas ours does not - we only assume the invariant marginal distribution of $p_{\zz_{c}} $ does not vary over domains. 

\subsection{Identifiability of the Joint Distribution in the Shared Space}
\label{subsec:joint_distribution_identifiability}

Equipped with the identifiability of the changing components and the invariant subspace, we now discuss how these can result in the identifiability of the joint distribution $p_{ \xx, \yy | \uu^{\cT} } $ for the target domain $\uu^{\cT}$ with a classifier trained on source domain labels.

Since $g$ is partially invertible and partially identifiable (Theorem~\ref{thm:changing_part_identifiability} and Theorem~\ref{thm:block_identifiability_content}), we can recover $ \zz_{c}, \zz_{s} $ from observation $ \xx $ in any domain, that is, we can partially identify the joint distribution $ p_{ \xx, \zz_{c}, \zz_{s} | \uu } $.
As $ \zz_{s} = f_{\uu} (\tilde{\zz}_{s}) $ and $f_{\uu}$ is constrained to be component-wise monotonic and invertible, we can further identify $ p_{ \xx, \zz_{c}, \tilde{\zz}_{s} | \uu } $.
Also note that the label $\yy$ is conditionally independent of $\xx$ and especially $\uu$ given $\zz_{c}$ and $ \tilde{\zz}_{s} $ as shown in Figure~\ref{fig:generating_process}, which means that $ \zz_{c} $ and $ \tilde{\zz}_{s} $ capture all the information in $\xx$ that is relevant to $\yy$.
Therefore, given $ \zz_{c}, \tilde{\zz}_{s} $ we can make predictions in the target domain with the predictor $ p_{ \yy | \zz_{c}, \tilde{\zz}_{s} } $, and this predictor can be learned on source domains, as $ p_{ \xx, \zz_{c}, \zz_{s} | \uu } $ is identifiable as discussed above.
This is exactly the goal of UDA.
In fact, we can identify the joint distribution $ p_{\xx, \yy | \uu^{\cT}}  $ - with the identified $ p_{ \xx, \zz_{c}, \tilde{\zz}_{s} | \uu^{\cT} } $ for the target domain, we can derive $ p_{ \xx, \yy | \uu^{\cT} } = \int_{\zz_{c}} \int_{\tilde{\zz}_{s}} p_{ \xx, \zz_{c}, \tilde{\zz}_{s} | \uu^{\cT} } \cdot p_{ \yy | \zz_{c}, \tilde{\zz}_{s} } d\zz_{c} d\tilde{\zz}_{s} $.

The intuition is that by resorting to the high-level invariant part $ \tilde{\zz}_{s} $ we are able to account for the influence of the domains and thus obtain a classifier useful to all domains participating in the estimator learning, which includes the target domain in the UDA setup.
This reasoning motivates our algorithm design presented in Section~\ref{sec:implementation}.

\section{Partially-Identifiable VAE for Domain Adaptation}
\label{sec:implementation}

\label{subsec:invariant_subspace_identifiability}
\begin{figure}[th]
    \centering
    \includegraphics[width=0.36\textwidth]{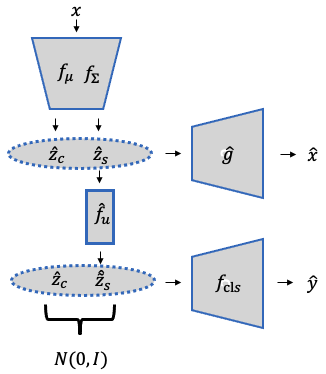}
    \caption{Diagram of our proposed method, \ours. We first apply the VAE encoder $(f_{\mu}, f_{\Sigma})$ to encode $\xx$ into $ (\hat{\zz}_{c}, \hat{\zz}_{s}) $, which is further fed into the decoder $\hat{g}$ for reconstruction. In parallel, the changing part $ \hat{\zz}_{s} $ is passed through the flow model $f_{\uu}$ to recover the high-level invariant variable $\hat{\tilde{\zz}}_{s}$. We use $ (\hat{\zz}_{c}, \hat{\tilde{\zz}}_{s}) $ for classification with the classifier $f_{\text{cls}}$ and for matching $\cN(\mathbf{0}, \mI)$ with a KL loss.}
    \label{fig:diagram}
\end{figure}

In this section, we discuss how to turn the theoretical insights in Section~\ref{sec:theory} into a practical algorithm, \ours\ (\textbf{i}dentifiable \textbf{M}ulti-\textbf{S}ource \textbf{D}omain \textbf{A}daptation), by leveraging expressive deep learning models.
As instructed by the theory, we can estimate the underlying causal generating process in Figure~\ref{fig:generating_process} and recover the ground-truth latent variables up to certain subspaces.
This in turn allows for optimal classification across domains.
In the following, we describe how to estimate each component in the casual generating process ($g$, $ p_{\zz_{s} | \uu} $, $ p_{\zz_{c}} $, $ p_{ \yy | \zz_{c}, \tilde{\zz}_{s} } $) under the VAE framework.
The architecture diagram is shown in Figure~\ref{fig:diagram}.

We use the VAE encoder to obtain the posterior $q( \hat{\zz} | \xx )$, which essentially estimates the inverse of the mixing function $ g^{-1} $. 
Specifically, we parameterize the mean and the covariance matrix diagonal of posterior $ q_{ f_{ \mu }, f_{ \Sigma }  }( \hat{\zz} | \xx) $ with multilayer perceptron's (MLP) $f_{\mu} $ and $ f_{\Sigma} $ respectively (Equation~\ref{eq:posterior}). 
The VAE decoder $ \hat{g} $, which is also an MLP, processes the estimated latent variable $ \hat{\zz} $ and estimates the input data $\xx$ (Equation~\ref{eq:reconstruction}).
\begin{align}
    &\hat{\zz} \sim q_{ f_{ \mu }, f_{ \Sigma }  }( \hat{\zz} | \xx): = \mathcal{N} ( f_{ \mu } ( \xx ), f_{ \Sigma } ( \xx ) ), \label{eq:posterior}
    \\
    &\hat{\xx} = \hat{g} (\hat{\zz}). \label{eq:reconstruction}
\end{align}

As indicated in Figure~\ref{fig:generating_process}, the changing part $\zz_{s}$ of the latent variable $\zz$ contains the information of a specific domain $\uu$, that is, $ \zz_{s}  = f_{\uu} (\tilde{\zz}_{s})$.
To enforce the minimal change property, we assume in Section~\ref{sec:theory} that domain influence is component-wise monotonic.
We estimate the domain influence function $f_{\uu}$ by flow-based architectures~\cite{dinh2017density,huang2018neural,durkan2019neural} for each domain $\uu$, i.e., $ \hat{\zz}_{s} = \hat{f}_{\uu} (\hat{\tilde{\zz}}_{s}) $.
Benefiting from the invertibility of $ \hat{f}_{\uu} $, we can obtain $ \hat{\tilde{\zz}}_{s} $ from $ \hat{\zz}_{s} $ sampled from Equation~\ref{eq:posterior} by inverting the estimated function $ \hat{f}^{-1}_{\uu} $:
\begin{align} \label{eq:inverting_flow}
    \hat{\tilde{\zz}}_{s} = \hat{f}^{-1}_{\uu} (\hat{\zz}_{s}).
\end{align}

Overall, the VAE loss can be employed to enforce the aforementioned generating process:  
\begin{flalign} \label{eq:vae_loss}
    \mathcal{L}_{ \mathrm{VAE} } (f_{\mu}, f_{\Sigma}, \hat{f}_{\uu}, \hat{g}) 
    &= \mathbb{E}_{
            (\mathbf{x}, \mathbf{u})
        } \; \mathbb{E}_{
                \tilde{\mathbf{z}} \sim q_{f_{\mu}, f_{\Sigma}} 
            } \frac{1}{2} \| \xx - \hat{\xx} \|^{2} && \\\nonumber
            & \quad + \beta \cdot 
                D ( q_{ f_{\mu}, f_{\Sigma}, \hat{f}_{\uu} } ( \hat{\tilde{\zz}} | \xx ) || p ( \tilde{\zz} ) )
            , &&
\end{flalign}
where $\beta$ is a control factor introduced in $\beta$-VAE~\cite{higgins2016beta}.
We use $ q_{ f_{\mu}, f_{\Sigma}, \hat{f}_{\uu} } ( \hat{\tilde{\zz}} | \xx ) $ to denote the posterior of $ \tilde{\zz} $ which can be obtained by executing Equation~\ref{eq:posterior} and Equation~\ref{eq:inverting_flow}. 
We choose $ p ( \tilde{\mathbf{z}} ) $ as a standard Gaussian distribution $\mathcal{N} ( \mathbf{0}, \mathbf{I} ) $, following the independence assumption (Assumption~\ref{assumption:factorizable_prior}).
The KL divergence in Equation~\ref{eq:vae_loss} is tractable, thanks to the Gaussian form of $ q_{f_{\mu}, f_{\Sigma}} $ (Equation~\ref{eq:posterior}) and the tractability of the Jacobian determinant of the flow model $ \hat{f}_{\uu} $.

Lastly, we enforce the causal relation between $\yy$ and $ (\zz_{c}, \tilde{\zz}_{s}) $ in our estimated generating process (Figure~\ref{fig:generating_process}).
This can be implemented as a classification branch: 
\begin{align} \label{eq:estimate_y}
    \hat{\yy} = f_{\mathrm{cls}} (\hat{\zz}_{c}, \hat{\tilde{\zz}}_{s}),
\end{align}
where $ f_{\mathrm{cls}} $ is parameterized by a MLP. 
In sources domains, we directly optimize a cross-entropy loss:
\begin{align} \label{eq:cross_entropy}
    \cL_{ \mathrm{cls} } ( f_{ \mathrm{cls} }, f_{\mu}, f_{\Sigma}, \hat{f}_{\uu}, \hat{g} ) = -\mathbb{E}_{ \hat{\mathbf{z}}_{c}, \hat{ \mathbf{z} }_{s}, \, \yy } \Big[ \yy \cdot \log \left( \hat{\yy} \right) \Big],
\end{align}
where $\yy$ are one-hot ground-truth label vectors available in the source domains.
In the target domain, we enforce this relation by maximizing the mutual information $ \mI( (\hat{\zz}_{c}, \hat{\tilde{\zz}}_{s}); \hat{\yy} ) $ which amounts to minimizing the conditional entropy $ \mH (\hat{\yy} | \hat{\zz}_{c}, \hat{\tilde{\zz}}_{s})$~\citep{wang2020learning,li2021t}.
Doing so facilitates learning a latent representation space adhering to the generating process.
\begin{align} \label{eq:entropy}
    \cL_{\mathrm{ent}} ( f_{ \mathrm{cls} }, f_{\mu}, f_{\Sigma}, \hat{f}_{\uu}, \hat{g} ) = -\mathbb{E}_{ \hat{\mathbf{z}}_{c}, \hat{ \mathbf{z} }_{s} } \Big[ \hat{\yy} \cdot \log \left( \hat{\yy} \right) \Big].
\end{align}

\begin{algorithm}[h]
\caption{Training \ours}\label{alg:training}
\begin{algorithmic}
\REQUIRE $(\xx, \yy, \uu)$ from source domains and $ ( \xx, \uu ) $ from the target domain.
\REQUIRE $ f_{\mu}$, $f_{\Sigma}$, $\hat{g}$, $\hat{f}_{\uu}$, and $ f_{\mathrm{cls}} $.
    \STATE 1. get the latent representation $ \hat{\zz} \sim q_{ f_{ \mu }, f_{ \Sigma }  }( \hat{\zz} | \xx) $ in Equation~\ref{eq:posterior};
    \STATE 2. reconstruct the observation $ \hat{\xx} = \hat{g} (\hat{\zz}) $ in Equation~\ref{eq:reconstruction};
    \STATE 3. recover the high-level invariant part $ \hat{\tilde{\zz}}_{s} = \hat{f}_{\uu}^{-1} (\hat{\zz}) $ in Equation~\ref{eq:inverting_flow};
    \STATE 4. estimate the class label $ \hat{\yy} = f_{\mathrm{cls}} (\hat{\zz}_{c}, \hat{\tilde{\zz}}_{s}) $ in Equation~\ref{eq:estimate_y};
    \STATE 5. compute and minimize $ \cL $ according to Equation~\ref{eq:vae_loss}, Equation~\ref{eq:entropy}, Equation~\ref{eq:cross_entropy}, and Equation~\ref{eq:objective}.
\end{algorithmic}
\end{algorithm}

Overall, our loss function is
\begin{align} \label{eq:objective}
    \mathcal{ L } ( f_{ \text{cls} }, \hat{f}_{\mu}, f_{\Sigma}, f_{\mathbf{u}}, \hat{g} ) 
    = \mathcal{ L }_{\text{cls}} + \alpha_1 \mathcal{ L }_{\text{ent}} + \alpha_2 \mathcal{L}_{\text{VAE}},
\end{align}
where $\alpha_1$, $\alpha_2$ are hyper-parameters that balance the loss terms. 
The complete procedures are shown in Algorithm~\ref{alg:training}.

\section{Experiments on Synthetic Data}
\label{sec:synthetic_experiments}

In this section, we present synthetic data experiments to verify Theorem~\ref{thm:changing_part_identifiability} and Theorem~\ref{thm:block_identifiability_content} in practice.

\subsection{Experimental Setup}

\paragraph{Data generation.}
We generate synthetic data following the generating process of $\xx$ in Equation~\ref{eq:data_generating_process}.
We work with latent variables $\zz$ of $4$ dimensions with $n_{c} = n_{s} = 2$.
We sample $\zz_{c} \sim \mathcal{N}(\mathbf{0}, \mathbf{I})$ and $ \zz_{s} \sim \mathcal{N}( \mu_{\uu}, \sigma^{2}_{\uu} \mathbf{I} ) $ where for each domain $\uu$ we sample $ \mu_{\uu} \sim \mathrm{Unif}( -4, 4 ) $ and $ \sigma^{2}_{\uu} \sim \mathrm{Unif}( 0.01, 1 ) $.
We choose $g$ to be a $2$-layer MLP with the Leaky-ReLU activation function \footnote{
    Our experimental design closely follows the practice employed in prior nonlinear ICA work~\cite{hyvarinen2016unsupervised,hyvarinen2019nonlinear}.
}.

\paragraph{Evaluation metrics.}
To measure the component-wise identifiability of the changing components, we compute Mean Correlation Coefficient (MCC) between $\zz_{s}$ and $\hat{\zz}_{s}$ on a test dataset.
MCC is a standard metric in the ICA literature. A higher MCC indicates a higher extent of identifiability, and MCC reaches $1$ when latent variables are perfectly component-wise identifiable.
To measure the block-wise identifiability of the invariant subspace, we follow the experimental practice of the work~\cite{von2021self} and compute the $R^2$ coefficient of determination between $\zz_{c}$ and $\hat{\zz}_{c}$, where $R^2=1$ suggests there is a one-to-one mapping between the two.
As iVAE and $\beta$-VAE do not specify the invariant and the changing subspace in their estimation. At each of their evaluation, we test all possible partitions of their estimated latent variables and report the one with the highest overall score (Avg.).  
We repeat each experiment over $3$ random seeds.

\subsection{Results and Discussion.}
From Table~\ref{tab:synthetic_experiments}, we can observe that both $R^{2}$ and MCC of our method grow along with the number of domains with the peak at $d=9$ where both the invariant and the changing part attain high identifiability scores.
This confirms our theory and the intuition that a certain number of domains are necessary for identifiability.
Notably, we can observe that decent identifiability can be attained with a relatively small number of domains (e.g. $5$), which sheds light on why \ours would work even with only a moderate number of domains.
We visualize the disentanglement in Figure~\ref{fig:synthetic_scatterplots}.

\begin{figure}[t]
     \centering
     \includegraphics[width=0.5\textwidth]{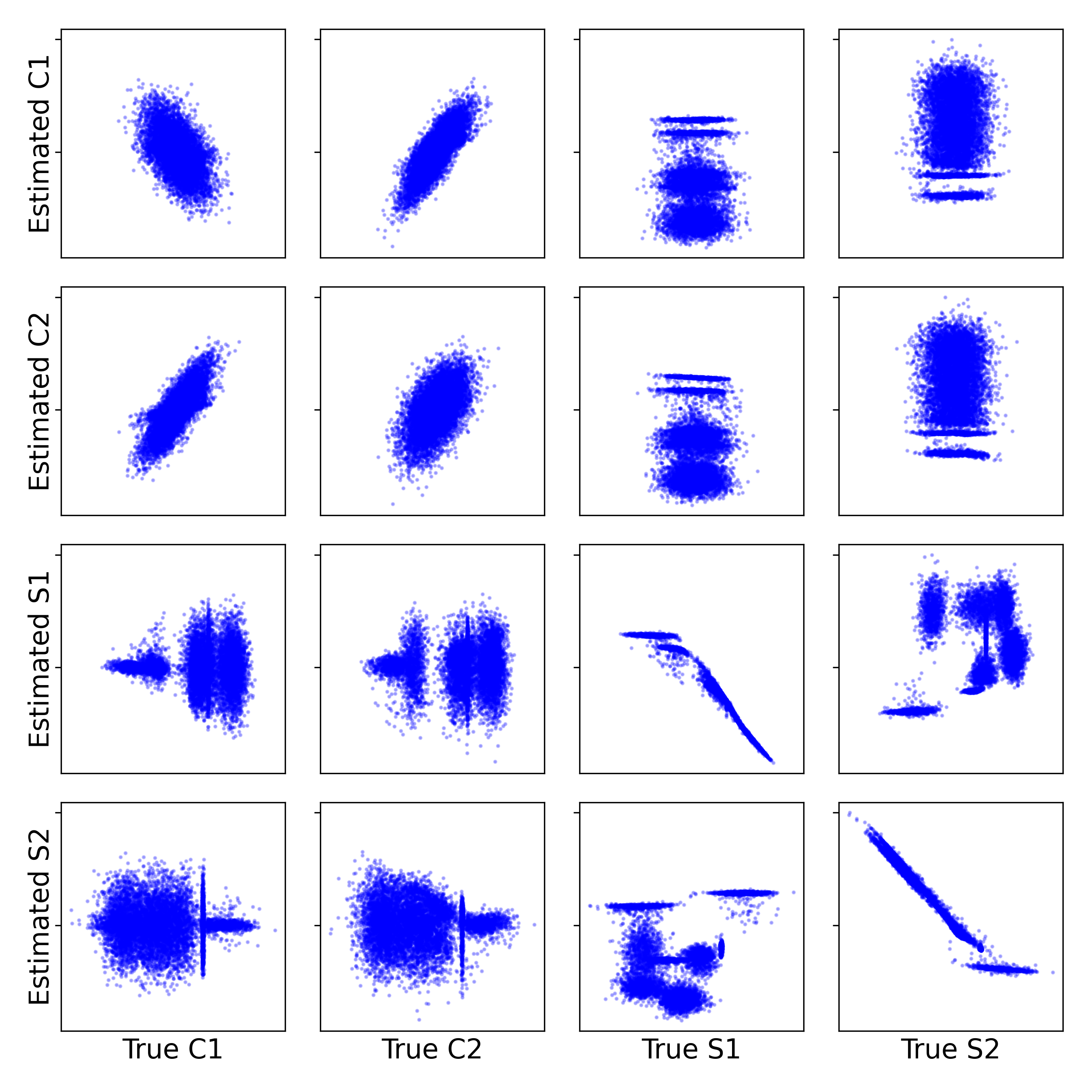}
    \caption{
        \textbf{The scatter plot for the true and the estimated components from our method trained with 9 domains.} 
        The $y$ (resp. $x$) axe of each subplot corresponds to a specific estimated (resp. true) latent variable.
        We can observe that the changing components can be identified in a component-wise manner (subplots "Estimated S" and "True S"), which verifies our Theorem~\ref{thm:changing_part_identifiability}.
        The true invariant components can be (partial) identified within its own subspace (subplots "Estimated C" and "True C") while influencing the changing components minimally, adhering to Theorem~\ref{thm:block_identifiability_content}.
    }
    \label{fig:synthetic_scatterplots}
\end{figure}


\begin{table}[t]
    \centering
    \resizebox{\columnwidth}{!}{%
    \begin{tabular}{c|cccc}
        \hline
         & ours ($d=3$) & ours ($d=5$) & ours ($d=7$) & ours ($d=9$)  \\
        \hline
        MCC & $ 0.67 \pm 0.12 $ & $ 0.80 \pm 0.15 $ & $ 0.90 \pm 0.12 $ & $\mathbf{ 0.91  \pm 0.09 } $ \\ 
        $R^{2}$ & $ 0.59 \pm 0.05 $ &  $ 0.74 \pm 0.09 $ & $ 0.84 \pm 0.05 $ & $\mathbf{ 0.85 \pm 0.06} $  \\
        Avg. & $ 0.63 \pm 0.08$ & $ 0.77 \pm 0.12 $ & $ 0.87 \pm 0.09 $ & $ \mathbf{0.88 \pm 0.08 }$ \\
        \hline
    \end{tabular}%
    }
    \caption{\text{Identifiability on synthetic data}: $d$ denotes the number of domains and Avg. stands for the average of MCC and $R^{2}$.}
    \label{tab:synthetic_experiments}
\end{table}

\section{Experiments on Real-world Data}
\label{sec:real_world_experiments}

\begin{table*}[]
    \centering
    \begin{tabular}{c|cccc|c}
    \hline
         Methods&$\rightarrow$ Art& $\rightarrow$ Cartoon& $\rightarrow$ Photo& $\rightarrow$ Sketch&Avg  \\
         \hline
         Source Only \cite{he2016deep} & 74.9 $ \pm $ 0.88& 72.1$ \pm $0.75 &94.5$ \pm $0.58 &64.7$ \pm $1.53 &76.6\\
    DANN \cite{ganin2016domain} & 81.9$ \pm $1.13 &77.5$ \pm $1.26 &91.8$ \pm $1.21 &74.6$ \pm $1.03 &81.5\\
    MDAN \cite{zhao2018adversarial}& 79.1$ \pm $0.36 &76.0$ \pm $0.73 &91.4$ \pm $0.85 &72.0$ \pm $0.80 &79.6\\
    WBN \cite{mancini2018boosting}& 89.9$ \pm $0.28 &89.7$ \pm $0.56 &97.4$ \pm $0.84 &58.0$ \pm $1.51 &83.8\\
    MCD \cite{saito2018maximum}& 88.7$ \pm $1.01 &88.9$ \pm $1.53 &96.4$ \pm $0.42 &73.9$ \pm $3.94 &87.0\\
    M3SDA \cite{peng2019moment} & 89.3$ \pm $0.42 &89.9$ \pm $1.00 &97.3$ \pm $0.31 &76.7$ \pm $2.86 &88.3\\
        CMSS \cite{yang2020curriculum}&  88.6 $ \pm $0.36& 90.4$ \pm $ 0.80 &96.9$ \pm $0.27 &82.0$ \pm $0.59 &89.5\\
        LtC-MSDA \cite{wang2020learning} &90.19 &90.47&97.23 &81.53  &89.8\\
         T-SVDNet \cite{li2021t} &90.43 &90.61&\textbf{98.50} &85.49  &91.25\\
         iMSDA (Ours) & $\mathbf{93.75}\pm 0.32$ & $\mathbf{92.46}\pm 0.23$ & $98.48\pm 0.07$ &  $\mathbf{89.22}\pm 0.73$ &  \textbf{93.48} \\ \hline
    \end{tabular}
    \caption{Classification results on PACS. We employ Resnet-18 as our encoder backbone. Most baseline results are taken from \cite{yang2020curriculum}. We choose $\alpha_{1}=0.1$ and $\alpha_{2}=5e-5$. The latent space is partitioned with $n_{s}=4$ and $n=64$.}
    \label{tab:pacs}
\end{table*}

\begin{table*}[]
    \centering
    \begin{tabular}{c|cccc|c}
    \hline
  Models & $\rightarrow$ Art& $\rightarrow$ Clipart & $\rightarrow$ Product & $\rightarrow$ Realworld &Avg\\ \hline
Source Only \cite{he2016deep}  &64.58$ \pm $0.68 &52.32$ \pm $0.63 &77.63$ \pm $0.23 &80.70$ \pm $0.81 &68.81\\
DANN \cite{ganin2016domain} &64.26$ \pm $0.59 &58.01$ \pm $1.55 &76.44$ \pm $0.47 &78.80$ \pm $0.49 &69.38\\
DANN+BSP \cite{chen2019transferability} & 66.10$ \pm $0.27 &61.03$ \pm $0.39 &78.13$ \pm $0.31 &79.92$ \pm $0.13 &71.29\\
DAN \cite{long2015learning} &68.28$ \pm $0.45 &57.92$ \pm $0.65 &78.45$ \pm $0.05 &81.93$ \pm $0.35 &71.64\\
MCD \cite{saito2018maximum} &67.84$ \pm $0.38 &59.91$ \pm $0.55 &79.21$ \pm $0.61 &80.93$ \pm $0.18 &71.97\\ \hline
M3SDA \cite{peng2019moment} &66.22$ \pm $0.52 &58.55$ \pm $0.62 &79.45$ \pm $0.52 &81.35$ \pm $0.19 &71.39\\
DCTN \cite{xu2018deep} &66.92$ \pm $0.60 &61.82$ \pm $0.46 &79.20$ \pm $0.58 &77.78$ \pm $0.59 &71.43\\
MIAN-$\gamma$ \cite{park2021information}&69.88$ \pm $0.35 &\textbf{64.20$ \pm $0.68} &80.87$ \pm $0.37 &81.49$ \pm $0.24 &74.11 \\
iMSDA (Ours) & $\mathbf{75.4} \pm 0.86$ & $61.4 \pm 0.73$ & $\mathbf{83.5} \pm 0.22$ & $\mathbf{84.47} \pm 0.38$ & $\mathbf{76.19}$ \\ \hline
    \end{tabular}
    \caption{Classification results on Office-Home. We employ Resnet-50 as our encoder backbone. Baseline results are taken from \cite{park2021information}. We choose $\alpha_{1}=0.1$ and $\alpha_{2}=1e-4$. The latent space is partitioned with $n_{s}=4$ and $n=128$.}
    \label{tab:office}
\end{table*}

\begin{table}
    \centering 
    \resizebox{\columnwidth}{!}{%
    \begin{tabular}[h]{ |c | c c c | }
        \hline
            & $\cL_{\text{cls}}$ & $\cL_{\text{cls}} + \alpha_{1} \cL_{\text{ent}}$ & $\cL$ \\
        \hline
        $\rightarrow$P & $97.2 \pm 0.01$ & $98.4 \pm 0.26$ & $ 98.48 \pm 0.07 $ \\
        \hline 
        $\rightarrow$A & $81.1 \pm 0.73$ & $93.44 \pm 0.31$ & $ 93.75 \pm 0.32 $ \\ 
        \hline
        $\rightarrow$C & $79.2 \pm 0.17$ & $ 89.03 \pm 2.4$& $ 92.46 \pm 0.23 $ \\
        \hline 
        $\rightarrow$S & $ 70.9 \pm 0.69$ & $87.29 \pm 0.8$ & $ 89.22 \pm 0.73 $ \\
        \hline 
        Avg. & $82.11$ & $92.01$ & $ 93.48 $ \\
        \hline
    \end{tabular}
    }%
    \caption{Ablation study on PACS dataset. We study the impact of individual loss terms. Except for the ablated terms, the hyper-parameters are identical to those in Table~\ref{tab:pacs}.}
    \label{tab:ablation_loss_pacs}
\end{table}

\begin{table}
    \centering 
    \resizebox{\columnwidth}{!}{%
    \begin{tabular}[h]{ |c | c c c | }
        \hline
            & $\cL_{\text{cls}}$ & $\cL_{\text{cls}} + \alpha_{1} \cL_{\text{ent}}$ & $\cL$ \\
        \hline
        $\rightarrow$Ar & $64.57 \pm 0.04$ & $74.23 \pm 0.66$ & $ 75.4 \pm 0.86 $ \\
        \hline 
        $\rightarrow$Pr & $77.77 \pm 0.03$ & $82.83 \pm 0.12$ & $ 83.5 \pm 0.22 $ \\ 
        \hline
        $\rightarrow$Cl & $47.17 \pm 0.20$ & $ 61.8 \pm 0.1$& $ 61.4 \pm 0.73 $ \\
        \hline 
        $\rightarrow$Rw & $ 79.03 \pm 0.08$ & $84.2 \pm 0.14$ & $ 84.47 \pm 0.38 $ \\
        \hline 
        Avg. & $67.2$ & $75.77$ & $ 76.19 $ \\
        \hline
    \end{tabular}
    }%
    \caption{Ablation study on Office-Home dataset. We study the impact of individual loss terms. Except for the ablated terms, the hyper-parameters are identical to those in Table~\ref{tab:office}.}
    \label{tab:ablation_loss_office_home}
\end{table}

\begin{table}
    \centering
    \begin{tabular}[h]{c| ccc }
        \hline
         & $ n_{s} = 8 $ & $ n_{s} = 4 $ & $n_{s} = 2$  \\
        \hline
        $\rightarrow$P & $98.40 \pm 0.1$ & $ 98.48 \pm 0.07$ & $98.34 \pm 0.07$ \\
        \hline 
        $\rightarrow$A & $ 92.15 \pm 1.74$ & $ 93.75 \pm 0.32 $ & $93.86 \pm 0.46$ \\
        \hline 
        $\rightarrow$C & $ 91.48 \pm 1.54 $ & $92.46 \pm 0.23$  & $92.18 \pm 0.34$\\
        \hline
        $\rightarrow$S & $86.15 \pm 5.2 $ & $ 89.22 \pm 0.73 $ & $	86.25 \pm 3.05$   \\
        \hline
        Avg   & $92.05 \pm 2.14 $ & $93.48 \pm 0.34 $ & $92.66 \pm 0.98$\\
        \hline
    \end{tabular}
    \caption{The impact of the changing dimension $n_{s}$: we test several values of $n_{s}$ on PACS dataset. The other configurations are identical to those in Table~\ref{tab:pacs}. }
    \label{tab:ablation_changing_components_pacs}
\end{table}

\subsection{Experimental Setup}
\paragraph{Datasets.}
We evaluate our proposed iMSDA on two benchmark domain adaptation datasets, namely Office-Home \cite{venkateswara2017deep} and
PACS \cite{li2017deeper}. 
Please see Appendix~\ref{sec:real_world_setup} for a detailed description.
In experiments, we designate one domain as the target domain and use all other domains as source domains.

\paragraph{Baselines.}
We compare our approach with state-of-the-art
methods to verify its effectiveness.  We compare with Source Only and single-source domain adaptation methods: DANN \cite{ganin2016domain}, MCD \cite{saito2018maximum}, DANN+BSP \cite{chen2019transferability}. We also compare our method with existing multi-source domain adaptation methods, including M3SDA \cite{peng2019moment}, CMSS \cite{yang2020curriculum}  LtC-MSDA \cite{wang2020learning}, MIAN-$\gamma$ \cite{park2021information} and T-SVDNet \cite{li2021t}.
Note that when using single-source adaptation methods, we pool all sources together as a single source and then apply the methods.
We defer implementation details to Appendix~\ref{sec:real_world_setup}.

\subsection{Results and Discussion}

\paragraph{PACS.}  
The results for PACS are presented in Table~\ref{tab:pacs}. 
We can observe that for the majority of the transfer directions, \ours\ outperforms the most competitive baseline by a considerable margin of $1.2 \%$ - $ 3 \% $. For the $\rightarrow$Phone task where it does not, the performance is within the margin of error compared to the strongest algorithm T-SVDNet.
Notably, when compared with the currently proposed T-SVDNet \cite{li2021t}, \ours\ achieves a significant performance gain on the challenging task $\rightarrow$Sketch. 
To aid understanding, we visualize the learned features in Figure~\ref{fig:tsne} (Appendix~\ref{sec:real_world_setup}), which shows that \ours\ can effectively align source and target domains while retaining discriminative structures.

\paragraph{Office-Home.} 
The results in Table \ref{tab:office} demonstrate that the superior performance of \ours\ on the Office-Home dataset.
\ours\ achieves a margin of roughly $3\%$ over other baselines on $3$ out of $4$ tasks, with an exception for $\rightarrow$Clipart.
Although \ours\ loses out to MIAN-$\gamma$ for the $\rightarrow$Clipart task, it is still on par with or superior to other baselines.

\subsection{Ablation studies}

\paragraph{The impact of individual loss term.}
Table~\ref{tab:ablation_loss_pacs} and Table~\ref{tab:ablation_loss_office_home} contain the ablation results for individual loss terms.
We can observe that $\cL_{\text{ent}}$ greatly boosts the model performance in both cases (around $10\%$ and $8.5\%$).
This shows the directly enforcing the relation between $(\hat{\zz}_{c}, \hat{\tilde{\zz}}_{s}) $ and $\yy$ is an effective solution if the change is insignificant over domains.
The inclusion of the alignment consistently benefits the model performance in each case, which justifies our theoretical insights and design choice.

\paragraph{The impact of the changing part dimension.}
Table~\ref{tab:ablation_changing_components_pacs} includes results on various numbers of changing components $n_{s}$.
We can observe the model performance degrades at both large $n_{s}$ (i.e.\ $8$) and small $n_{s}$ (i.e.\ $2$) choices.
This is coherent with our hypothesis: a large $n_{s}$ compromises the minimal change principle, leading to the undesirable consequences as discussed in Section~\ref{sec:motivation};
on the contrary, an overly small $n_{s}$ could prevent the model from handling the domain change, for lack of changing capacity, which also yields suboptimal performance.

\section{Conclusion}

It is not uncommon to assume observations of the real world are generated from high-level latent variables and thus the ill-posedness in the problem of UDA can be reduced to obtaining meaningful reconstructions of those latent variables and mapping distinct domains to a shared space for classification. \par
In this work, we show that under reasonable assumptions on the data-generating process, as well as leveraging the principle of minimality, we can obtain partial identifiability of the changing and invariant parts of the generating process. In particular, by introducing a high-level invariant latent variable that influences the changing variable and the corresponding label across domains, we show identifiability of the joint distribution $p_{ \xx, \yy | \uu^{\cT} } $ for the target domain $\uu^{\cT}$ with a classifier trained on source domain labels. Our proposed VAE combined with a flow model architecture learns disentangled representations that allow us to perform multi-source UDA with state-of-the-art results across various benchmarks. 

\paragraph{Acknowledgment}
We thank the anonymous reviewers for their constructive comments. This work was supported in part by the NSF-Convergence Accelerator Track-D award \#2134901 and by the National Institutes of Health (NIH) under Contract R01HL159805, and by a grant from Apple. The NSF or NIH is not responsible for the views reported in this paper.


\bibliography{references}
\bibliographystyle{configurations/icml2022}

\newpage
\appendix
\onecolumn

\section{Appendix}

\subsection{Proof of the changing part identifiability} \label{sec:changing_part_proof}
We present the proof of the changing part identifiability (i.e., Theorem~\ref{thm:changing_part_identifiability}) in this section.

\renewcommand{\wvectorcontent}{%
    \mathbf{w}( \mathbf{z}_{s}, \mathbf{u} ) 
    = \Bigl(
        \frac{\partial q_{n_{c}+1} \left( z_{ n_{c} + 1 }, \mathbf{u} \right) }{ \partial z_{ n_{c} + 1 } }, \ldots, \frac{ \partial q_{n} \left( z_n, \mathbf{u} \right) }{ \partial z_{n} },
        \frac{ \partial^{2} q_{ n_{c} + 1 } \left( z_{ n_{c} + 1 }, \mathbf{u} \right) }{ \partial z_{n_{c} + 1}^{2} }, \ldots, \frac{ \partial^{2} q_{ n } \left( z_{ n }, \mathbf{u} \right) }{ \partial z_{n}^{2} }%
    \Bigr)
}

\stheory*

\begin{proof}

We start from the matched marginal distribution condition (i.e., Equation~\ref{eq:estimated_condition}) to develop the relation between $ \mathbf{z} $ and $ \hat{\mathbf{z}} $ as follows: $ \forall \mathbf{u} \in \mathcal{U} $,
\begin{align} \label{eq:marginal_matching_consenquence}
    p_{ \hat{\mathbf{x}} | \mathbf{u} } = p_{ \mathbf{x} | \mathbf{u} }
    \iff
    p_{ \hat{g} (\hat{\mathbf{z}}) | \mathbf{u} } = p_{ g (\mathbf{z}) | \mathbf{u} }
    \iff p_{ g^{-1} \circ \hat{g}(\hat{\mathbf{z}}) | \mathbf{u} } | \mathbf{J}_{ g^{-1} } | = p_{ \zz | \mathbf{u} } | \mathbf{J}_{ g^{-1} } |
    \iff p_{ h(\hat{\mathbf{z}}) | \mathbf{u} } = p_{ \mathbf{z} | \mathbf{u} },
\end{align}
where $ \hat{g}^{-1}: \mathcal{Z} \to \mathbf{\mathcal{Z}} $ denotes the estimated invertible generating function,  and $h:= g^{-1} \circ \hat{g} $ is the transformation between the true latent variable and the estimated one.
$|\mathbf{J}_{g^{-1}}|$ stands for the absolute value of Jacobian matrix determinant of $ g^{-1} $.
Note that as both $ \hat{g}^{-1} $ and $ g $ are invertible, $ |\mathbf{J}_{g^{-1}}| \neq 0 $ and $h $ is invertible.


According to the independence relations in the data-generating process in Equation \ref{eq:data_generating_process} and A2, we have
\begin{align*}
    p_{\mathbf{z} | \mathbf{u}} (\mathbf{z} | \mathbf{u}) = \prod_{i=1}^{n} p_{z_{i} | \mathbf{u}} (z_{i}) ; \qquad
    p_{\hat{\mathbf{z}} | \mathbf{u}} ( \hat{\mathbf{z}} | \mathbf{u}) = \prod_{i=1}^{n} p_{\hat{z}_{i} | \mathbf{u}} (\hat{z}_{i})
\end{align*}
Rewriting with the notation $ q_{i}:= \log p_{ z_{i} | \mathbf{u} } $ and $ \hat{q}_{i}:= \log p_{ \hat{z}_{i} | \mathbf{u} }  $ yields
\begin{align*}
    \log p_{ \mathbf{z} | \mathbf{u} }(\mathbf{z} | \mathbf{u}) =  \sum_{i=1}^{n} q_{i} (z_{i},\mathbf{u}) ; \qquad
    \log p_{ \hat{\mathbf{z}} | \mathbf{u} }( \hat{\mathbf{z}} | \mathbf{u}) =  \sum_{i=1}^{n} \hat{q}_{i} (\hat{z}_{i},\mathbf{u})
\end{align*}

Applying the change of variables to Equation~\ref{eq:marginal_matching_consenquence} yields 
\begin{align} \label{eq:log_transform}
    p_{\mathbf{z}| \mathbf{u}} = p_{ \hat{\zz} | \mathbf{u} } \cdot |\mathbf{J}_{h^{-1}}|
    \iff
    \sum_{i=1}^{n} q_{i} (z_{i},\mathbf{u}) + \log|\mathbf{J}_{h}| = \sum_{i=1}^{n} \hat{q}_{i} (\hat{z}_{i},\mathbf{u})
\end{align}
where $\mathbf{J}_{h^{-1}}$ and $\mathbf{J}_{h}$ are the Jacobian matrix of the transformation associated with $h^{-1}$ and $h$ respectively.

For ease of exposition, we adopt the following notation:
\begin{align*}
    h'_{i, (k)} := \frac{\partial z_{i}}{\partial \hat{z}_{k}},
    \qquad & h''_{i, (k, q)} := \frac{\partial^2 z_{i}}{\partial \hat{z}_{k} \partial \hat{z}_{q}}; \\
    \eta'_{i} ( z_{i}, \mathbf{u} ) := \frac{
    \partial q_{i} ( z_{i},\mathbf{u})
    }{
        \partial z_{i}
    }, 
    \qquad & \eta''_{i} ( z_{i}, \mathbf{u} ) := \frac{
    \partial^{2} q_{i} ( z_{i},\mathbf{u})
    }{
        (\partial z_{i})^{2}
    }.
\end{align*}

Differentiating both sides of Equation~\ref{eq:log_transform} twice w.r.t.\ $ \hat{z}_{k} $ and $ \hat{z}_{q} $ where $ k, q \in [n] $ and $ k \neq q $ yields 

\begin{align} \label{eq:single_differential_equation}
    \sum_{i=1}^{n} \Big( 
        \eta''_{i} ( z_{i}, \mathbf{u} ) \cdot h'_{i, (k)} h'_{i, (q)}  
        + \eta'_{i} ( z_{i}, \mathbf{u} ) \cdot  h''_{i, (k,q)} 
    \Big) + \frac{\partial^2 \log|\mathbf{J}_{h}|}  {\partial \hat{z}_{k} \partial \hat{z}_{q} } = 0.
\end{align}
Therefore, for $ \mathbf{u} = \mathbf{u}_{0}, \dots, \mathbf{u}_{2n_{s}} $, we have $ 2 n_{s} + 1 $ such equations. 
Subtracting each equation corresponding to $ \mathbf{u}_{1}, \dots, \mathbf{u}_{2n_{s}} $ with the equation corresponding to $ \mathbf{u}_{0} $ results in $ 2n_{s} $ equations:
\begin{align}
    \sum_{i=n_{c}+1}^{n} \Big( 
       (\eta''_{i} ( z_{i}, \mathbf{u}_{j} ) - \eta''_{i} ( z_{i}, \mathbf{u}_{0} ) ) \cdot h'_{i, (k)} h'_{i, (q)}  
        + ( \eta'_{i} ( z_{i}, \mathbf{u}_{j} ) - \eta'_{i} ( z_{i}, \mathbf{u}_{0} ) ) \cdot  h''_{i, (k,q)} 
    \Big) = 0,
\end{align}
where $ j = 1, \dots 2 n_{s} $.
Note that as the content distribution is invariant to domains i.e., $ p_{\mathbf{z}_{c}} = p_{\mathbf{z}_{c} | \mathbf{u}} $, we have $ \eta''_{i} ( z_{i}, \mathbf{u}_{j} ) = \eta''_{i} ( z_{i}, \mathbf{u}_{j'} ) $ and $ \eta'_{i} ( z_{i}, \mathbf{u}_{j} ) = \eta'_{i} ( z_{i}, \mathbf{u}_{j'} ), \forall j,  j' $. 
Hence only the style components $ i = n_{c} + 1, \dots, n $ remain in the summation of each equation.

Under the linear independence condition in Assumption A3, the linear system is a $ 2n_{s} \times 2n_{s} $ full-rank system. Therefore, the only solution is $h'_{i, (k)} h'_{i, (q)} = 0$ and $h''_{i, (k,q)} = 0$ for $ i = n_{c}+1, \dots, n $ and $ k, q \in [n], k \neq q $.

As $ h(\cdot) $ is smooth over $\mathcal{Z}$, its Jacobian can written as:
\begin{equation}
    \scalebox{1.25}{$
        \mathbf{J}_{h} =
        \begin{bmatrix}
        \mathbf{A}:= \frac{\partial \zz_{c} }{\partial \hat{\zz}_{c} }
        & \rvline & \mathbf{B}:= \frac{\partial \zz_{c} }{\partial \hat{\zz}_{s} } \\
        \hline
        \mathbf{C}:= \frac{\partial \zz_{s} }{\partial \hat{\zz}_{c} } & \rvline &
        \mathbf{D}:= \frac{\partial \mathbf{z}_{s} }{\partial \hat{\zz}_{s} },
        \end{bmatrix}
    $}
\end{equation}

Note that $h'_{i, (k)} h'_{i, (q)} = 0$ implies that for each $i = n_{c}+1, \dots, n $, $ h'_{i, (k)} \neq 0 $ for at most one element $k \in [n]$. 
Therefore, there is only at most one non-zero entry in each row indexed by $ i = n_{c}+1, \dots, n $ in the Jacobian matrix $ \mathbf{J}_{h} $.
Further, the invertibility of $ h (\cdot)$ necessitates $ \mathbf{J}_{h} $ to be full-rank, which implies that there is exactly one non-zero component in each row of matrices $ \mathbf{C} $ and $ \mathbf{D} $.
\\\\
Since for every $i \in \{n_{c}+1, \ldots, n \}$, $z_{i}$ has changing distributions over $ \mathbf{u} $ and all components $ \hat{z}_{k} $ for $ i \in \{1, \ldots, n_{c} \} $ (i.e., $\hat{\zz}_{c}$) have invariant distributions over $ \mathbf{u} $, we can deduce that $ \mathbf{C} = \mathbf{0} $ and the only non-zero entry $ \frac{ z_{i} }{ \hat{z}_{k} } $ must reside in $ \mathbf{D} $ with $ k \in \{n_{c}+1, \ldots, n \} $.
Therefore, for each true variable in the changing part $ z_{i}, i \in \{n_{c}+1, \ldots, n \} $, there exists one estimated variable in the changing part $ \hat{z}_{k},  k \in \{n_{c}+1, \ldots, n \} $ such that $ z_{i} = h'_{i} ( \hat{z}_{k} ) $.
Further, because $\mathbf{J}_{h}$ is full-rank and $ \mathbf{C} $ is a zero matrix, the block matrix corresponding to $ \frac{\partial \hat{\zz}_{s}}{ \partial \zz_{c} } $ in the Jacobian $ \mJ_{h^{-1}} $ is also a zero matrix, which implies that $\hat{\zz}_{s}$ is not influenced by $\zz_{c}$.
Therefore, there exists an one-to-one mapping between $\zz_{s}$ and $\hat{\zz}_{s}$, which implies that $ h'_{i}  (\cdot) $ is invertible for each $ i \in \{n_{c}+1, \ldots, n \} $.
Thus, the changing components $ \zz_{s} $ are identified up to permutation and component-wise invertible transformations.

\end{proof}

\subsection{Proof of the invariant part identifiability} \label{sec:proof_invariant_part}
In the section, we present the proof of the invariant part identifiability (i.e., Theorem~\ref{thm:block_identifiability_content}).

\ctheory*

\begin{proof} \label{proof:block_identifiability_content}
    We divide our proof into four steps to aid understanding.
    \\\\
    In Step 1, we leverage properties of the generating process (i.e., Equation~\ref{eq:data_generating_process}) and the marginal distribution matching condition (i.e., Equation~\ref{eq:estimated_condition}) to express the marginal invariance with the indeterminacy transformation $ \bar{h}: \mathcal{Z} \to \mathcal{Z} $ between the estimated and the true latent variable.
    The introduction of $ \bar{h}(\cdot) $ allows us to formalize the block-identifiability condition.
    \\\\
    In Step 2 and Step 3, we show that the estimated content variable $ \hat{\mathbf{z}}_{c} $ does not depend on the true style variable $ \mathbf{z}_{s} $, that is, $ \bar{h}_{c}(\mathbf{z}) $ does not depend on the input $ \mathbf{z}_{s} $.
    To this end, in Step 2, we derive its equivalent statements, which can ease the rest of the proof and avert technical issues (e.g.,  sets of zero probability measures).
    In Step 3, we prove the equivalent statement by contradiction. Specifically, we show that if $\hat{\mathbf{z}}_{c}$ depended on $\mathbf{z}_{s}$, the invariance derived in Step 1 would break.
    \\\\
    In Step 4, we use the conclusion in Step 3, the smooth and bijective properties of $h(\cdot)$, and the conclusion in Theorem~\ref{thm:changing_part_identifiability}, to show the invertibility of the indeterminacy function between the content variables, i.e., the mapping $\hat{\mathbf{z}}_{c} = \bar{h}_{c}(\mathbf{z}_{c})$ being invertible.
    \\\\
    \textbf{Step 1.}
    As the data generating process of the learned model $ ( \hat{ g }, \, p_{ \hat{\mathbf{z}}_{s} } , \, p_{ \hat{\mathbf{z}}_{s} | \mathbf{ \mathbf{u} } } ) $ (Equation~\ref{eq:data_generating_process}) establishes the independence between the generating process $ \hat{\mathbf{z}}_{c} \sim p_{\hat{\mathbf{z}}_{c}} $ \footnote{
        Note that in this proof, we decorate quantities with $ \hat{\cdot} $ to indicate that they are associated with the estimated model $ ( \hat{g}, \, p_{\hat{\zz}_{c}}, \, p_{ \hat{\zz}_{s} | \uu } ) $.
    }
    and $ \mathbf{u} $, it follows that for any $ A_{\mathbf{z}_{c}} \subseteq \mathcal{Z}_{c} $,
    \begin{align} 
        \Pb{ \{ \hat{g}^{-1}_{1:n_{c}} ( \hat{\mathbf{x}} ) \in A_{\mathbf{z}_{c}} \} | \{ \mathbf{u} = \mathbf{u}_{1} \} } 
        &= \Pb{ \{ \hat{g}^{-1}_{1:n_{c}} ( \hat{\mathbf{x}} ) \in A_{\mathbf{z}_{c}} \} | \{ \mathbf{u} = \mathbf{u}_{2} \} }, \quad \forall \mathbf{u}_{1}, \mathbf{u}_{2} \in \mathcal{U} \nonumber \\
        & \iff \nonumber \\
        \Pb{ \{ \hat{\mathbf{x}} \in (\hat{g}^{-1}_{1:n_{c}})^{-1} (A_{\mathbf{z}_{c}}) \} | \{ \mathbf{u} = \mathbf{u}_{1} \} } 
        &= \Pb{ \{ \hat{\mathbf{x}} \in (\hat{g}^{-1}_{1:n_{c}})^{-1} (A_{\mathbf{z}_{c}}) \} | \{ \mathbf{u} = \mathbf{u}_{2} \} }, \quad \forall \mathbf{u}_{1}, \mathbf{u}_{2} \in \mathcal{U} \label{eq:marginal_matching_learned}
    \end{align}
    where $ \hat{g}^{-1}_{1:n_{c}}: \mathcal{X} \to \mathcal{Z}_{c} $ denotes the estimated transformation from the observation to the content variable and $ (\hat{g}^{-1}_{1:n_{c}})^{-1} (A_{\mathbf{z}_{c}}) \subseteq \mathcal{X} $ is the pre-image set of $ A_{\mathbf{z}_{c}} $, that is, the set of estimated observations $ \hat{\mathbf{x}} $ originating from content variables $ \hat{\mathbf{z}}_{c} $ in $ A_{\mathbf{z}_{c}} $.

    Because of the matching observation distributions between the estimated model and the true model in Equation~\ref{eq:estimated_condition}, the relation in Equation~\ref{eq:marginal_matching_learned} can be extended to observation $\mathbf{x}$ from the true generating process $ (g, p_{\mathbf{z}_{c}}, p_{\mathbf{z}_{s} | \mathbf{u}}) $, i.e.,
    \begin{align}
        \Pb{ \{ \mathbf{x} \in (\hat{g}^{-1}_{1:n_{c}})^{-1} (A_{\mathbf{z}_{c}}) \} | \{ \mathbf{u} = \mathbf{u}_{1} \} } 
        & = \Pb{ \{ {\mathbf{x}} \in (\hat{g}^{-1}_{1:n_{c}})^{-1} (A_{\mathbf{z}_{c}}) \} | \{ \mathbf{u} = \mathbf{u}_{2} \} } \nonumber \\
        & \iff \nonumber \\
        \Pb{ \{ \hat{g}^{-1}_{1:n_{c}} (\mathbf{x}) \in A_{\mathbf{z}_{c}} \} | \{ \mathbf{u} = \mathbf{u}_{1} \} } 
        & = \Pb{ \{ \hat{g}^{-1}_{1:n_{c}} (\mathbf{x}) \in A_{\mathbf{z}_{c}} \} | \{ \mathbf{u} = \mathbf{u}_{2} \} }. \label{eq:equation_with_g}
    \end{align}

    Since $ g $ and $ \hat{g} $ are smooth and injective, there exists a smooth and injective $ \bar{h} = \hat{g}^{-1} \circ g: \mathcal{Z} \to \mathcal{Z} $. We note that by definition $\bar{h} = h^{-1} $ where $h$ is introduced in the proof of Theorem~\ref{thm:changing_part_identifiability} (Appendix~\ref{sec:changing_part_proof}).
    Expressing $ \hat{g}^{-1} = \bar{h} \circ g^{-1} $ and $ \bar{h}_{c} (\cdot):= \bar{h}_{1:n_{c}} (\cdot): \mathcal{Z} \to \mathcal{Z}_{c} $ in Equation~\ref{eq:equation_with_g} yields
    \begin{align}
        \Pb{ \{ \bar{h}_{c} ( \mathbf{z} ) \in A_{\mathbf{z}_{c}} \} | \{ \mathbf{u} = \mathbf{u}_{1} \} } 
        &= \Pb{ \{ \bar{h}_{c} ( \mathbf{z} ) \in A_{\mathbf{z}_{c}} \} | \{ \mathbf{u} = \mathbf{u}_{2} \} }, \nonumber \\
        & \iff \nonumber \\
        \Pb{ \{ \mathbf{z} \in \bar{h}_{c}^{-1} (A_{\mathbf{z}_{c}}) \} | \{ \mathbf{u} = \mathbf{u}_{1} \} } 
        &= \Pb{ \{ \mathbf{z} \in \bar{h}_{c}^{-1} (A_{\mathbf{z}_{c}}) \} | \{ \mathbf{u} = \mathbf{u}_{2} \} }, \nonumber \\
        & \iff \nonumber \\
        \int_{ \mathbf{z} \in \bar{h}_{c}^{-1} ( A_{ \mathbf{z}_{c} } ) } p_{ \mathbf{z} | \mathbf{u} } ( \mathbf{z} | \mathbf{u}_{1} ) \, d\mathbf{z}
        &= \int_{ \mathbf{z} \in \bar{h}_{c}^{-1} ( A_{ \mathbf{z}_{c} } ) } p_{ \mathbf{z} | \mathbf{u} } ( \mathbf{z} | \mathbf{u}_{2} ) \, d\mathbf{z},
        \label{eq:equal_preimage_zc}
    \end{align}
    where $ \bar{h}_{c}^{-1} (A_{\mathbf{z}_{c}}) = \{ \mathbf{z} \in \mathcal{Z}: \bar{h}_{c} (\mathbf{z} ) \in A_{\mathbf{z}_{c}} \} $ is the pre-image of $ A_{\mathbf{z}_{c}} $, i.e.\ those latent variables containing content variables in $ A_{\mathbf{z}_{c}} $ after the indeterminacy transformation $ h $.
    \\\\
    Based on the generating process in Equations~\ref{eq:data_generating_process}, we can re-write Equation~\ref{eq:equal_preimage_zc} as follows: $ \forall A_{\mathbf{z}_{c}} \subseteq \mathcal{Z}_{c} $, 
    \begin{align}
        \begin{split}
            \int_{ [\mathbf{z}_{c}^{\top}, \mathbf{z}_{s}^{\top}]^{\top} \in \bar{h}_{c}^{-1} ( A_{ \mathbf{z}_{c} } ) } p_{\mathbf{z}_{c}} ( \mathbf{z_{c}} ) \left(
                p_{\mathbf{z}_{s} | \mathbf{u}} ( \mathbf{z_{s}} | \mathbf{u}_{1} ) - p_{\mathbf{z}_{s} | \mathbf{u}} ( \mathbf{z_{s}} | \mathbf{u}_{2} )
            \right) d\mathbf{z}_{s} d\mathbf{z}_{c} = 0. \label{eq:equation_to_test}
        \end{split}
    \end{align}

    \textbf{Step 2.}
    In order to show the block-identifiability of $\mathbf{z}_{c}$, we would like to prove that $ \mathbf{z}_{c} := \bar{h}_{c} ( [\mathbf{z}_{c}^{\top}, \mathbf{z}_{s}^{\top}]^{\top}) $ does not depend on $ \mathbf{z}_{s} $. 
    To this end, we first develop one equivalent statement (i.e.\ Statement 3 below) and prove it in later step instead. 
    By doing so, we are able to leverage the full-supported density function assumption to avert technical issues.
    \begin{itemize}
        \item Statement 1: $ \bar{h}_{c} ( [\mathbf{z}_{c}^{\top}, \mathbf{z}_{s}^{\top}]^{\top}) $ does not depend on $ \mathbf{z}_{s} $.
        \item Statement 2: $ \forall \mathbf{z}_{c} \in \mathcal{Z}_{c} $, it follows that $ \bar{h}_{c}^{-1} ( \mathbf{z}_{c} ) = B_{\mathbf{z}_{c}} \times \mathcal{Z}_{s} $ where $ B_{\mathbf{z}_{c}} \neq \emptyset $ and $ B_{\mathbf{z}_{c}} \subseteq \mathcal{Z}_{c} $. 
        \item Statement 3: $ \forall \mathbf{z}_{c} \in \mathcal{Z}_{c} $, $ r \in \mathbb{R}^{+} $, it follows that $ \bar{h}_{c}^{-1} ( \mathcal{B}_{r} (\mathbf{z}_{c}) ) = B^{+}_{\mathbf{z}_{c}} \times \mathcal{Z}_{s} $ where $ \mathcal{B}_{r} (\mathbf{z}_{c}):= \{ \mathbf{z}'_{c} \in \mathcal{Z}_{c}: || \mathbf{z}'_{c} - \mathbf{z}_{c} ||^{2} < r \} $, $ B^{+}_{\mathbf{z}_{c}} \neq \emptyset $, and $ B^{+}_{\mathbf{z}_{c}} \subseteq \mathcal{Z}_{c} $.
    \end{itemize}
    Statement 2 is a mathematical formulation of Statement 1. 
    Statement 3 generalizes singletons $ \mathbf{z}_{c} $ in Statement 2 to open, non-empty balls $ \mathcal{B}_{r} (\mathbf{z}_{c}) $.
    Later, we use Statement 3 in Step 3 to show the contraction to Equation~\ref{eq:equation_to_test}.
    \\\\
    Leveraging the continuity of $ \bar{h}_{c} (\cdot) $, we can show the equivalence between Statement 2 and Statement 3 as follows.
    We first show that Statement 2 implies Statement 3.
    $ \forall \mathbf{z}_{c} \in \mathcal{Z}_{c}, r \in \mathbb{R}^{+} $, $ \bar{h}_{c}^{-1} ( ( \mathcal{B}_{r} (\mathbf{z}_{c}) ) ) = \cup_{ \mathbf{z}_{c}' \in \mathcal{B}_{r} (\mathbf{z}_{c}) } h^{-1}_{c} ( \mathbf{z}_{c}' ) $. 
    Statement 2 indicates that every participating sets in the union satisfies $ h^{-1}_{c} (\mathbf{z}_{c}') = B_{\mathbf{z}_{c}}' \times \mathcal{Z}_{s} $, thus the union $ \bar{h}_{c}^{-1} ( ( \mathcal{B}_{r} (\mathbf{z}_{c}) ) ) $ also satisfies this property, which is Statement 3. 
    \\\\
    Then, we show that Statement 3 implies Statement 2 by contradiction.
    Suppose that Statement 2 is false, then $ \exists \hat{\mathbf{z}}_{c} \in \mathcal{Z}_{c}$ such that there exist $ \hat{\mathbf{z}}_{c}^{B} \in \{ \mathbf{z}_{1:n_{c}}: \mathbf{z} \in \bar{h}_{c}^{-1} (\hat{\mathbf{z}}_{c}) \} $ and $ \hat{\mathbf{z}}_{s}^{B} \in \mathcal{Z}_{s} $ resulting in $ \bar{h}_{c} ( \hat{\mathbf{z}}^{B} ) \neq \hat{\mathbf{z}}_{c} $ where $ \hat{\mathbf{z}}^{B} = [(\hat{\mathbf{z}}_{c}^{B})^{\top}, (\hat{\mathbf{z}}_{s}^{B})^{\top}]^{\top} $.
    As $ \bar{h}_{c} (\cdot) $ is continuous, there exists $ \hat{r} \in \mathbb{R}^{+} $ such that $ \bar{h}_{c} ( \hat{\mathbf{z}}^{B} ) \not\in \mathcal{B}_{\hat{r}} ( \hat{\mathbf{z}}_{c} ) $.
    That is, $ \hat{\mathbf{z}}^{B} \not\in h^{-1}_{c} ( \mathcal{B}_{\hat{r}} ( \hat{\mathbf{z}}_{c} ) ) $.
    Also, Statement 3 suggests that $ h^{-1}_{c} ( \mathcal{B}_{\hat{r}} ( \hat{\mathbf{z}}_{c} ) ) = \hat{B}_{\mathbf{z}_{c}} \times \mathcal{Z}_{s} $.
    By definition of $\hat{\mathbf{z}}^{B}$, it is clear that $ \hat{\mathbf{z}}^{B}_{1:n_{c}} \in \hat{B}_{\mathbf{z}_{c}} $.
    The fact that $ \hat{\mathbf{z}}^{B} \not\in h^{-1}_{c} ( \mathcal{B}_{\hat{r}} ( \hat{\mathbf{z}}_{c} ) ) $ contradicts Statement 3.
    Therefore, Statement 2 is true under the premise of Statement 3. 
    We have shown that Statement 3 implies Statement 2.
    In summary, Statement 2 and Statement 3 are equivalent, and therefore proving Statement 3 suffices to show Statement 1.
    \\\\
    \textbf{Step 3.}
    In this step, we prove Statement 3 by contradiction.
    Intuitively, we show that if $\bar{h}_{c}(\cdot)$ depended on $\hat{z}_{s}$, the preimage $ \bar{h}_{c}^{-1} ( \mathcal{B}_{r} (\mathbf{z}_{c}) ) $ could be partitioned into two parts (i.e.\ $ B^{*}_{ \mathbf{z} }$ and $ \bar{h}_{c}^{-1} ( A^{*}_{ \mathbf{z}_{c} } ) \setminus B^{*}_{ \mathbf{z} } $ defined below).
    The dependency between $\bar{h}_{c}(\cdot)$ and $\hat{z}_{s}$ is captured by $ B^{*}_{ \mathbf{z} } $, which would not emerge otherwise.
    In contrast, $ \bar{h}_{c}^{-1} ( A^{*}_{ \mathbf{z}_{c} } ) \setminus B^{*}_{ \mathbf{z} } $ also exists when $\bar{h}_{c}(\cdot)$ does not depend on $\hat{z}_{s}$.
    We evaluate the invariance relation Equation~\ref{eq:equation_to_test} and show that the integral over $ \bar{h}_{c}^{-1} ( A^{*}_{ \mathbf{z}_{c} } ) \setminus B^{*}_{ \mathbf{z} } $ (i.e.\ $T_{1}$) is always $0$, however, the integral over $ B^{*}_{ \mathbf{z} } $ (i.e.\ $T_{2}$) is necessarily non-zero, which leads to the contraction with Equation~\ref{eq:equation_to_test} and thus shows the $\bar{h}_{c}(\cdot)$ cannot depend on $\hat{z}_{s}$.
    \\\\
    First, note that because $ \mathcal{B}_{r} (\mathbf{z}_{c}) $ is open and $ \bar{h}_{c} (\cdot) $ is continuous, the pre-image $ \bar{h}_{c}^{-1} ( \mathcal{B}_{r} (\mathbf{z}_{c}) ) $ is open. 
    In addition, the continuity of $ h(\cdot) $ and the matched observation distributions $ \forall \mathbf{u}' \in \mathcal{U}, \; \Pb{ \{ \mathbf{x} \in A_{\mathbf{x}} \} | \{ \mathbf{u} = \mathbf{u}' \} } = \Pb{ \{ \hat{\mathbf{x}} \in A_{\mathbf{x}} \} | \{ \mathbf{u} = \mathbf{u}' \} } $ lead to $ h(\cdot) $ being bijection as shown in \cite{klindt2021nonlinear},which implies that $ \bar{h}_{c}^{-1} ( \mathcal{B}_{r} (\mathbf{z}_{c}) ) $ is non-empty.
    Hence, $ \bar{h}_{c}^{-1} ( \mathcal{B}_{r} (\mathbf{z}_{c}) ) $ is both non-empty and open.
    Suppose that $ \exists A_{\mathbf{z}_{c}}^{*}:= \mathcal{B}_{r^{*}} (\mathbf{z}^{*}_{c}) $ where $ \mathbf{z}^{*}_{c} \in \mathcal{Z}_{c} $, $ r^{*} \in \mathbb{R}^{+} $, such that $ B^{*}_{ \mathbf{z} }:= \{ \mathbf{z} \in \mathcal{Z}: \mathbf{z} \in \bar{h}_{c}^{-1} ( A_{\mathbf{z}_{c}}^{*} ), \, \{ \mathbf{z}_{1:n_{c}} \} \times \mathcal{Z}_{s} \not\subseteq \bar{h}_{c}^{-1} ( A_{\mathbf{z}_{c}}^{*} ) \} \neq \emptyset $.  
    Intuitively, $ B^{*}_{ \mathbf{z} } $ contains the partition of the pre-image $ \bar{h}_{c}^{-1}(A_{\mathbf{z}_{c}}^{*}) $ that the style part $ \mathbf{z}_{n_{c}+1:n} $ cannot take on any value in $ \mathcal{Z}_{s} $.  
    Only certain values of the style part were able to produce specific outputs of indeterminacy $ \bar{h}_{c}(\cdot) $.
    Clearly, this would suggest that $ \bar{h}_{c}(\cdot) $ depends on $ \mathbf{z}_{c} $.
    \\
    To show contraction with Equation~\ref{eq:equation_to_test}, we evaluate the LHS of Equation~\ref{eq:equation_to_test} with such a $ A_{\mathbf{z}_{c}}^{*} $:
    \begin{align*}
        &\int_{ [\mathbf{z}_{c}^{\top}, \mathbf{z}_{s}^{\top}]^{\top} \in \bar{h}_{c}^{-1} ( A^{*}_{ \mathbf{z}_{c} } ) } p_{\mathbf{z}_{c}} ( \mathbf{z_{c}} ) \left(
            p_{\mathbf{z}_{s} | \mathbf{u}} ( \mathbf{z_{s}} | \mathbf{u}_{1} ) - p_{\mathbf{z}_{s} | \mathbf{u}} ( \mathbf{z_{s}} | \mathbf{u}_{2} )
        \right) d\mathbf{z}_{s} d\mathbf{z}_{c} \\ 
        &= \underbrace{
                \int_{ [\mathbf{z}_{c}^{\top}, \mathbf{z}_{s}^{\top}]^{\top} \in \bar{h}_{c}^{-1} ( A^{*}_{ \mathbf{z}_{c} } ) \setminus B^{*}_{ \mathbf{z} } } p_{\mathbf{z}_{c}} ( \mathbf{z_{c}} ) \left(
                p_{\mathbf{z}_{s} | \mathbf{u}} ( \mathbf{z_{s}} | \mathbf{u}_{1} ) - p_{\mathbf{z}_{s} | \mathbf{u}} ( \mathbf{z_{s}} | \mathbf{u}_{2} )
            \right) d\mathbf{z}_{s} d\mathbf{z}_{c} 
        }_{ T_{1} }\\
        &\qquad+ \underbrace{
            \int_{ [\mathbf{z}_{c}^{\top}, \mathbf{z}_{s}^{\top}]^{\top} \in B^{*}_{ \mathbf{z} } } p_{\mathbf{z}_{c}} ( \mathbf{z_{c}} ) \left(
                p_{\mathbf{z}_{s} | \mathbf{u}} ( \mathbf{z_{s}} | \mathbf{u}_{1} ) - p_{\mathbf{z}_{s} | \mathbf{u}} ( \mathbf{z_{s}} | \mathbf{u}_{2} )
            \right) d\mathbf{z}_{s} d\mathbf{z}_{c}
        }_{ T_{2} }.
    \end{align*} 

    We first look at the value of $ T_{1} $.
    When $ \bar{h}_{c}^{-1} ( A^{*}_{ \mathbf{z}_{c} } ) \setminus B^{*}_{ \mathbf{z} } = \emptyset $, $T_{1}$ evaluates to $ 0 $.
    Otherwise, by definition we can rewrite $ \bar{h}_{c}^{-1} ( A^{*}_{ \mathbf{z}_{c} } ) \setminus B^{*}_{ \mathbf{z} } $ as $ C^{*}_{ \mathbf{z}_{c}} \times \mathcal{Z}_{s} $ where $ C^{*}_{ \mathbf{z}_{c} } \neq \emptyset $ and $ C^{*}_{ \mathbf{z}_{c} } \subset \mathcal{Z}_{c} $.
    With this expression, it follows that
    \begin{align*}
        & \int_{ [\mathbf{z}_{c}^{\top}, \mathbf{z}_{s}^{\top}]^{\top} \in C^{*}_{ \mathbf{z}_{c} } } p_{\mathbf{z}_{c}} ( \mathbf{z_{c}} ) \left(
            p_{\mathbf{z}_{s} | \mathbf{u}} ( \mathbf{z_{s}} | \mathbf{u}_{1} ) - p_{\mathbf{z}_{s} | \mathbf{u}} ( \mathbf{z_{s}} | \mathbf{u}_{2} )
        \right) d\mathbf{z}_{s} d\mathbf{z}_{c} \\
        & \qquad = \int_{ \mathbf{z}_{c} \in C^{*}_{ \mathbf{z}_{c} } } p_{\mathbf{z}_{c}} ( \mathbf{z_{c}} ) \int_{ \mathbf{z}_{s} \in \mathcal{Z}_{s} } \left(
                p_{\mathbf{z}_{s} | \mathbf{u}} ( \mathbf{z_{s}} | \mathbf{u}_{1} ) - p_{\mathbf{z}_{s} | \mathbf{u}} ( \mathbf{z_{s}} | \mathbf{u}_{2} )
            \right) d\mathbf{z}_{s} d\mathbf{z}_{c} \\
        & \qquad = \int_{ \mathbf{z}_{c} \in C^{*}_{ \mathbf{z}_{c} } } p_{\mathbf{z}_{c}} ( \mathbf{z_{c}} ) \left( 1 - 1 \right) d\mathbf{z}_{c} = 0.
    \end{align*}
    Therefore, in both cases $T_{1}$ evaluates to $0$ for $A^{*}_{ \mathbf{z}_{c} }$.
    \\\\
    Now, we address $T_{2}$.
    As discussed above, $ \bar{h}_{c}^{-1} ( A^{*}_{\mathbf{z}_{c}} ) $ is open and non-empty.
    Because of the continuity of $\bar{h}_{c}(\cdot)$, $ \forall \mathbf{z}_{B} \in B_{\mathbf{z}}^{*} $, there exists $ r(\mathbf{z}_{B}) \in \mathbb{R}^{+}$ such that $ \mathcal{B}_{r(\mathbf{z}_{B})} (\mathbf{z}_{B}) \subseteq B_{\mathbf{z}}^{*} $.
    As $ p_{ \mathbf{z} | \mathbf{u} } ( \mathbf{z} | \mathbf{u} ) > 0 $ over $ (\mathbf{z}, \mathbf{u}) $, we have $ \Pb{ \{ \mathbf{z} \in B^{*}_{ \mathbf{z} } \} | \{ \mathbf{u} = \mathbf{u}' \} } \ge \Pb{ \{ \mathbf{z} \in \mathcal{B}_{r(\mathbf{z}_{B})} (\mathbf{z}_{B}) | \{ \mathbf{u} = \mathbf{u}' \} \} } > 0 $ for any $ \mathbf{u}' \in \mathcal{U}$.
    Assumption A4 indicates that $ \exists \mathbf{u}_{1}^{*}, \mathbf{u}_{2}^{*} $, such that
    \begin{align*}
        T_{2} := \int_{ [\mathbf{z}_{c}^{\top}, \mathbf{z}_{s}^{\top}]^{\top} \in B^{*}_{ \mathbf{z} } } p_{\mathbf{z}_{c}} ( \mathbf{z_{c}} ) \left(
            p_{\mathbf{z}_{s} | \mathbf{u}} ( \mathbf{z_{s}} | \mathbf{u}_{1}^{*} ) - p_{\mathbf{z}_{s} | \mathbf{u}} ( \mathbf{z_{s}} | \mathbf{u}_{2}^{*} )
        \right) d\mathbf{z}_{s} d\mathbf{z}_{c} \neq 0.
    \end{align*}
    Therefore, for such $ A^{*}_{\mathbf{z}_{c}} $, we would have $ T_{1} + T_{2} \neq 0 $ which leads to contradiction with Equation~\ref{eq:equation_to_test}.
    We have proved by contradiction that Statement 3 is true and hence Statement 1 holds, that is, $ \bar{h}_{c} (\cdot) $ does not depend on the style variable $ \mathbf{z}_{s} $. 
    \\\\
    \textbf{Step 4.}
    With the knowledge that $ \bar{h}_{c} (\cdot) $ does not depend on the style variable $ \mathbf{z}_{s} $, we now show that there exists an invertible mapping between the true content variable $\mathbf{z}_{c}$ and the estimated version $ \hat{\mathbf{z}}_{c} $.

    As $ \bar{h}(\cdot) $ is smooth over $\mathcal{Z}$, its Jacobian can written as:
    \begin{equation}
        \scalebox{1.25}{$
            \mathbf{J}_{\bar{h}} =
            \begin{bmatrix}
            \mathbf{A}:= \frac{\partial \hat{\mathbf{z}}_{c} }{\partial \mathbf{z}_{c} }
            & \rvline & \mathbf{B}:= \frac{\partial \hat{\mathbf{z}}_{c} }{\partial \mathbf{z}_{s} } \\
            \hline
            \mathbf{C}:= \frac{\partial \hat{\mathbf{z}}_{s} }{\partial \mathbf{z}_{c} } & \rvline &
            \mathbf{D}:= \frac{\partial \hat{\mathbf{z}}_{s} }{\partial \mathbf{z}_{s} },
            \end{bmatrix}
        $}
    \end{equation}
    where we use notation $ \hat{\mathbf{z}}_{c} = \bar{h}(\mathbf{z})_{1:n_{c}} $ and $ \hat{\mathbf{z}}_{s} = \bar{h}(\mathbf{z})_{n_{c}+1:n} $.
    On one hand, as we have shown that $ \hat{\mathbf{z}}_{c} $ does not depend on $ \mathbf{z}_{s} $, it follows that $ \mathbf{B} = \mathbf{0} $.
    On the other hand, as $ h(\cdot) $ is invertible over $\mathcal{Z}$, $\mathbf{J}_{\bar{h}}$ is non-singular. 
    Therefore, $ \mathbf{A} $ must be non-singular due to $ \mathbf{B} = \mathbf{0} $.
    We note that $ \mathbf{A} $ is the Jacobian of the function $ \bar{h}'_{c} (\mathbf{z}_{c}) := \bar{h}_{c} (\mathbf{z}) : \mathcal{Z}_{c} \to \mathcal{Z}_{c} $, which takes only the content part $ \mathbf{z}_{c} $ of the input $ \mathbf{z} $ into $ \bar{h}_{c} $.
    Also, the fact that the block matrix $\mathbf{B} = \mathbf{0} $ implies the block matrix corresponding to $ \frac{ \partial \zz_{c} }{ \partial \hat{\zz}_{s} } $ in the Jacobian matrix $ \mJ_{h} $ is also zero.
    Therefore, $\zz_{c}$ does not depend on $\hat{\zz}_{s} $.
    Together with the invertibility of $\bar{h}$, we can conclude that $\bar{h}'_{c}$ is invertible.
    Therefore, there exists an invertible function $ \bar{h}'_{c} $ between the estimated and the true content variables such that $ \hat{\mathbf{z}}_{c} = \bar{h}'_{c} ( \mathbf{z}_{c} ) $, which concludes the proof that $ \mathbf{z}_{c} $ is block-identifiable via $ \hat{g}^{-1}(\cdot) $. 
\end{proof}

\subsection{Synthetic Data Experiments}
\label{sec:synthetic_setups}
We provide additional details of our synthetic experiments in Section~\ref{sec:synthetic_experiments}.
\paragraph{Architecture}
For all methods in our synthetic data experiments, the VAE encoder and decoder are $6$-layer MLP's with a hidden dimension of $32$ and Leaky-ReLU ($\alpha=0.2$) activation functions.
For our method, we use component-wise spline flows \cite{durkan2019neural} with monotonic linear rational splines to modulate the change components. 
We use $8$ bins for the linear splines and set the bound to be $5$. 
The $\beta$-VAE shares the same VAE model with ours and the iVAE implementation is adopted from \cite{khemakhem2020ice}.

\paragraph{Training hyper-parameters}
We apply AdamW to train VAE and flow models for $100$ epochs. We use a learning rate of $0.002$ with a batch size of $128$. The weight decay parameter of AdamW is set to $0.0001$.
For VAE training, we set the $\beta$ parameter of KL loss term to $0.1$.

\subsection{Real-world Data Experiments}
\label{sec:real_world_setup}
We provide implementation details and additional visualization of real-world data experiments in Section~\ref{sec:real_world_experiments}.
\paragraph{Dataset description}
PACS \cite{li2017deeper} is a multi-domain dataset containing $9991$ images from 4 domains of different styles: Photo, Artpainting, Cartoon, Sketch. These domains share the same seven categories. Office-Home \cite{venkateswara2017deep} dataset consists of $4$ domains, with each domain containing images from 65 categories of everyday objects and a total of around $15, 500$ images. The domains include 1) Art: artistic depictions of objects; 2) Clipart: collection of clipart images; 3) Product: images of objects without a background; 4) Real-World: images of objects captured with a regular camera.

\paragraph{Implementation Details}
For the PACS dataset, we follow the protocols in \cite{yang2020curriculum,wang2020learning} and use ResNet-18 pre-trained on ImageNet. For the Office-Home dataset, we follow the protocols in \cite{park2021information} and use pre-trained ResNet-50
as our backbone network. For two datasets, we use SGD with Nesterov momentum with learning rate 0.01. Since it can be challenging to train VAE on high-resolution images, we use extracted features as our VAE input. In particular, we use the features after the last pooling layer in the ResNet as our VAE input. We use 2-layer fully connected network as the VAE encoder. For the flow model, we use Deep Sigmoidal Flow (DSF) \cite{huang2018neural} across all our experiments. 
For hyper-parameters, we fix $\alpha_{1}=0.1$ for all our experiments
and select $\alpha_2$ in $ [ 1e-5, 5e-5, 1e-4, 5e-4] $.

\paragraph{Feature Quality}
In addition, we visualize the learned features by our method in Figure \ref{fig:tsne} and make comparison with the Source Only baseline.
We can observe that, for the baseline the features in different classes are mixed, which renders the classification task significantly harder. In contrast, the features learned by iMSDA are more clustered and discriminative.   
\begin{figure}
     \centering
     \begin{subfigure}[b]{0.23\textwidth}
         \centering
         \includegraphics[width=\textwidth]{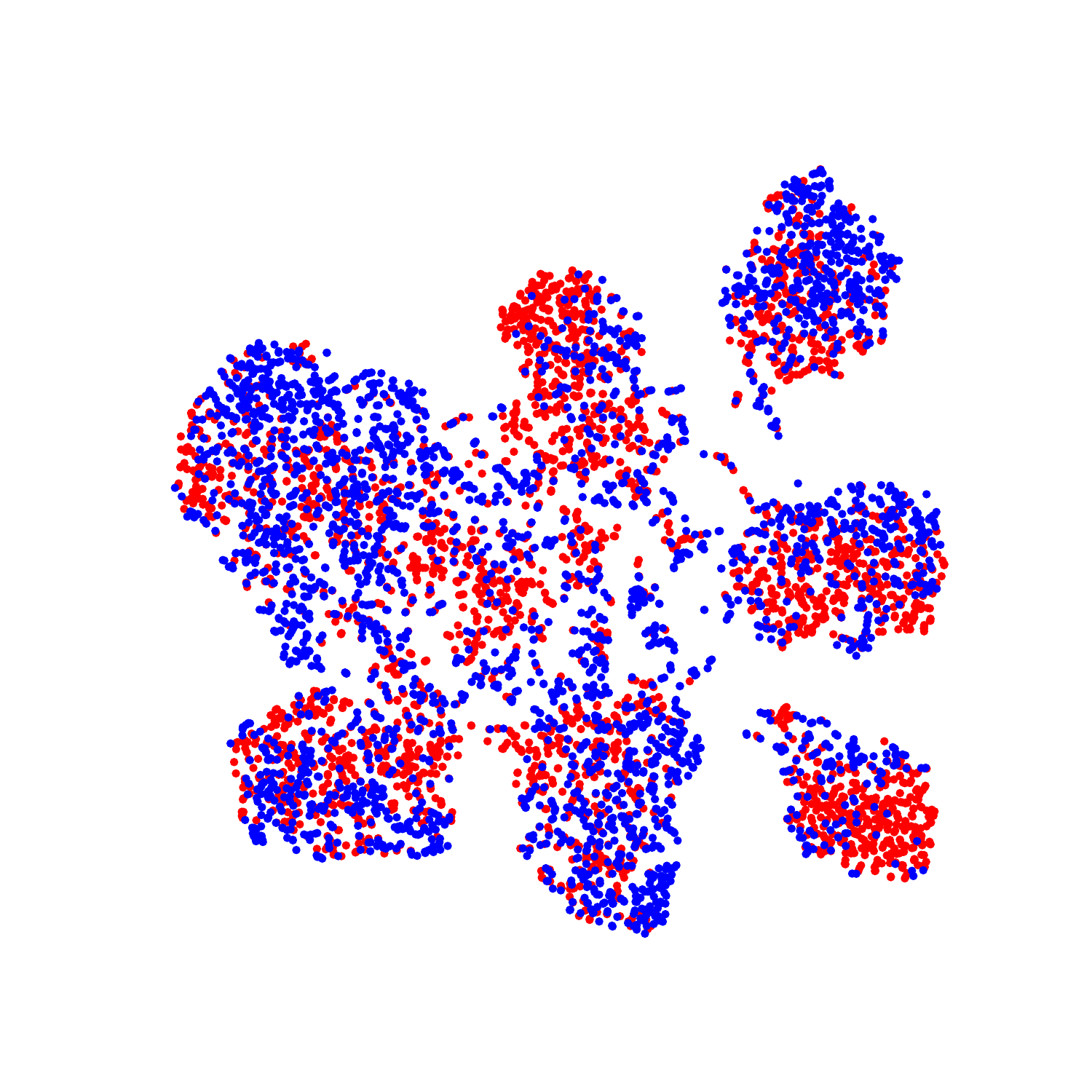}
         \caption{Source Only}
         \label{fig:synthetic_domain_number_baselines_mcc}
     \end{subfigure}
     \begin{subfigure}[b]{0.23\textwidth}
         \centering
         \includegraphics[width=\textwidth]{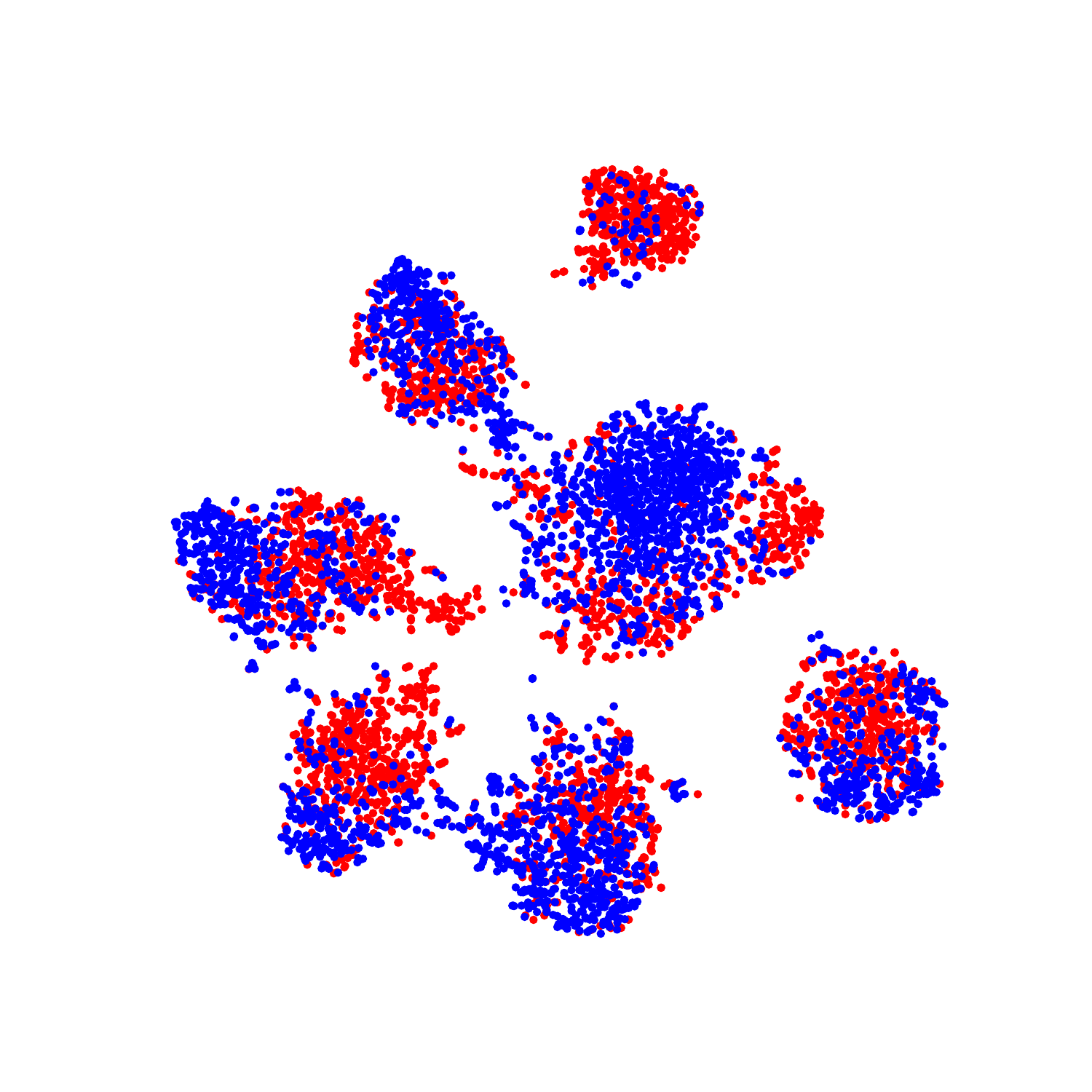}
         \caption{\ours}
         \label{fig:synthetic_domain_number_baselines_r2}
     \end{subfigure}
        \caption{
        {The t-SNE visualizations of learned features on the $\rightarrow$ Sketch task in PACS. Red: source domains, Blue: target domain}.
        }
        \label{fig:tsne}
\end{figure}


\end{document}